\documentclass{article}

\usepackage[letterpaper, hmargin=2.5cm, vmargin=2.5cm, bmargin=2.5cm]{geometry}

\usepackage{fancyhdr}
\newcommand{\shorttitle}{Domain-informed Adaptive Sampling for Generalizable PINNs in Metal AM via Conditional Flow Matching}

\usepackage[utf8]{inputenc}
\usepackage[T1]{fontenc}
\usepackage[section]{placeins}

\usepackage{amsfonts}
\usepackage{amsmath}
\usepackage{amssymb}
\usepackage{amsthm}
\usepackage{nicefrac}
\usepackage{bm}

\usepackage{graphicx}
\usepackage{xcolor}
\usepackage{tcolorbox}
\tcbuselibrary{breakable}
\usepackage{subcaption}
\usepackage{svg}
\usepackage[table]{xcolor}

\usepackage{booktabs}
\usepackage{multirow}

\usepackage{natbib}
\usepackage{hyperref}
\usepackage{url}
\hypersetup{
    colorlinks=true,
    linkcolor=blue!50!black,
    citecolor=blue!50!black,
    urlcolor=blue!50!black,
    pdftitle={Domain-informed Adaptive Sampling for Generalizable PINNs in Metal AM via Conditional Flow Matching},
    pdfsubject={cs.LG, physics.comp-ph},
    pdfauthor={Hyeonsu Lee, Jihoon Jeong},
    pdfkeywords={PINN, adaptive sampling, flow matching, additive manufacturing}
}

\usepackage{enumitem}
\usepackage{algorithm}
\usepackage{algorithmicx}
\usepackage{algpseudocode}

\renewenvironment{abstract}{%
  \centerline{\large\bfseries Abstract}\vspace{0.5em}%
  \begin{quote}%
}{%
  \end{quote}%
}

\newtheorem{assumption}{Assumption}[section]
\newtheorem{theorem}{Theorem}[section]
\newtheorem{remark}{Remark}[section]
\newtheorem{proposition}{Proposition}[section]

\newtheorem{lemma}{Lemma}[section]

\title{Domain-informed Adaptive Sampling for \\
       Generalizable
       PINNs in Metal Additive Manufacturing \\
       via Conditional Flow Matching}

\author{%
  Hyeonsu Lee,
  \hspace{1mm}Jihoon Jeong\thanks{Corresponding author. \texttt{jihoonjeong@tamu.edu}} \\
  Wm Michael Barnes '64 Department of Industrial and Systems Engineering\\
  Texas A\&M University\\
  College Station, TX 77843, USA\\
  \texttt{\{hslee, jihoonjeong\}@tamu.edu}
}                                                 

\date{}

\begin{document}
\maketitle
\begin{abstract}
Accurate thermal modeling is essential in metal additive manufacturing (AM) for understanding the process-structure-property chain. Physics-informed neural networks (PINNs) offer effective surrogate thermal modeling by minimizing physics-based residual losses at collocation points. However, prior works typically rely on manually-crafted, static collocation sampling strategies, which are neither principled nor scalable across process conditions, hindering their generalization capability. In this work, we provide theoretical analysis through empirical risk minimization, showing that process condition-aware adaptive sampling is strictly more favorable than conventional static sampling for generalization. Building on this insight, we propose an adaptive sampling strategy within a two-stage framework: (1) a conditional Flow Matching model that learns approximate high-residual distributions across different process conditions, and (2) a mixed sampling strategy combining this distribution with a domain-informed base distribution to generate adaptive collocation points for refining the PINN predictor. Experiments on metal AM numerical benchmarks demonstrate that our method consistently outperforms state-of-the-art PINN baselines, achieving an average 62.1\% reduction in relative $L_2$ error under an identical collocation budget, by capturing process-dependent heat dissipation regions often overlooked in the literature. To the authors' knowledge, this is the first adaptive sampling strategy for PINNs in metal AM, contributing to the enhanced generalization and broader applicability.
\end{abstract}

\section{Introduction}\label{sec:intro}

Metal additive manufacturing (AM) has received significant attention for its potential across diverse applications ranging from aerospace, automotive, and agriculture to construction and infrastructure \citep{frazier2014metal}. Due to the nature of metal AM processes, its primary advantage lies in its flexibility in realizing complex geometric variations and material compositions tailored to the unique requirements of each application \citep{tofail2018additive}. However, process instability and inconsistent part qualification are critical bottlenecks to its widespread industrial adoption \citep{svetlizky2021directed}.

To overcome these obstacles, accurate thermal field information is known to be essential in metal AM for understanding defect formation mechanisms and the process-structure-property chain \citep{smith2016linking,debroy2018additive}. As such, temperature field prediction under unseen process conditions enables effective process optimization, assuring product quality in metal AM \citep{ren2020thermal}.

Recently, Physics-Informed Neural Networks (PINNs) have been widely adopted as effective thermal surrogate models in metal AM. PINNs are trained by minimizing physics-based residuals at sampled \textit{collocation points} to satisfy governing PDEs, boundary conditions (BCs), and initial conditions (ICs), thereby yielding physically consistent, interpretable, and reliable predictions \citep{raissi2019physics}. Notably, PINNs enable mesh-free and data-efficient, or even data-free training \citep{zhu2021machine}, which has driven extensive research using PINNs for temperature field prediction in metal AM.

Previous progress in PINNs for metal AM can be broadly categorized into two perspectives: (1) improving performance by accounting for steep, highly transient thermal gradients in metal AM processes when training PINNs \citep{cooper2023error,peng2024multi,peng2025multi,yang2026physics}, and (2) enhancing generalization across unseen process conditions for rapid inference \citep{faegh2025physics,peng2025prediction,sajadi2025two}. Given that metal AM is especially known for its flexibility tailored to diverse applications through varying process conditions including material compositions, prior research has largely focused on advancing the generalizability of PINNs to align with the core flexibility that metal AM promises.

In this context, recent works have drawn significant attention to \textit{parametric PINNs} over conventional non-parametric PINN models, owing to their ability to amortize rapid inference across unseen process conditions. Unlike non-parametric PINNs, which often necessitate costly re-training or transfer learning-based adaptation for a single, specific target process condition \citep{liao2023hybrid}, parametric PINNs incorporate PDE parameters as additional inputs, enabling a single trained model to generalize across varying process conditions. For instance, \cite{hosseini2023single} developed a parametric PINN to generalize across different combinations of process parameters (e.g., laser power and scanning speed) and material properties within a limited range. \cite{yuan2025physics} proposed a parametric PINN for generalization to unseen process parameters alongside a transfer learning framework for rapid adaptation to different process conditions. Most recently, \cite{lee2026material} demonstrated that a decoupled parametric PINN architecture, along with physics-guided output scaling and a hybrid optimization strategy, achieves strong generalization performance across a wide range of material compositions including both in-distribution and challenging out-of-distribution (OOD) regimes. Their approach outperformed both conventional non-parametric PINNs and prior parametric PINNs, highlighting the potential of parametric PINNs as an advanced surrogate model for thermal modeling in metal AM.

These recent advances clearly show that parametric PINNs hold significant potential, not only in their capacity for rapid inference, but also in superior overall generalization performance across unseen process conditions. Nevertheless, to the best of the authors' knowledge, one critical limitation remains unaddressed in the literature:

\begin{quote}
Although \textit{collocation sampling strategies} are critical for both stable convergence and overall predictive performance in PINNs \citep{wang2022and}, previous research often relies on a single, manually-crafted collocation sampling strategy that is static and independent of process conditions \citep{liao2023hybrid}. While this strategy is informed by domain knowledge of metal AM and is shown to be adequate to some extent in terms of ensuring training stability and performance \citep{faegh2025physics}, it may be neither effective nor scalable, particularly for parametric PINNs, warranting further investigation.

Specifically, previous approaches statically concentrate collocation points in regions where steep and transient thermal gradients are expected (e.g., near the laser path), while remaining relatively coarse elsewhere. However, as evidenced by recent prior work \citep{lee2026material}, high-residual regions—where predictions deviate significantly from the physics-based solution—shift substantially depending on process conditions. For instance, materials with lower thermal conductivity localize intense heat near the laser path, creating sharply concentrated high-residual zones; conversely, higher thermal conductivity accelerates heat transport, causing high-residual regions to shift or expand toward broader heat dissipation zones. In this sense, the conventional static collocation sampling strategy that is independent of process conditions can be suboptimal and ultimately inhibits parametric PINNs from effectively generalizing across diverse process conditions.
\end{quote}

To this end, we introduce a domain-informed adaptive sampling for parametric PINNs in metal AM. Specifically, our main contributions are threefold:

\begin{enumerate}
\item \textbf{Theoretical Analysis on Generalization Error}: 
We provide a theoretical analysis of how collocation sampling impacts generalization across unseen process conditions within an empirical risk minimization framework. We demonstrate that PINN generalization error decomposes into optimization and statistical errors, where the latter is heavily governed by the variance of the loss estimator. In this regard, we prove that conventional static sampling strategies are strictly less effective in generalization than adaptive sampling strategies.
    
\item \textbf{Adaptive Sampling for PINNs in Metal AM using Conditional Flow Matching}: Grounded in our theoretical insights, we present an adaptive sampling framework for PINNs in metal AM using Flow Matching \citep{lipman2022flow}. By conditioning the generative model on both time steps and PDE parameters, our framework effectively approximates the process condition-dependent high-residual distribution and enables efficient generation of collocation points in a single pass. To the best of the authors' knowledge, this is the first adaptive sampling framework introduced for PINNs in metal AM.

\item \textbf{Integration of Metal AM Domain Knowledge}: We incorporate domain-informed physical priors both to warm-start the residual landscape that guides the generative model and to construct a base sampling distribution that is mixed with the learned proposal during training. This integration stabilizes and accelerates training for both the generator and the predictor.
\end{enumerate}

We validate the proposed method through extensive experiments on a widely used numerical benchmark, covering five representative metal alloys that span both in-distribution and OOD regimes. Under an identical collocation point budget, our method reduces the relative $L_2$ error of the current state-of-the-art parametric PINN baseline \citep{lee2026material} by up to $70.1\%$, purely by reallocating where collocation points are sampled.

Our qualitative analysis further shows that high-residual regions are not confined to the vicinity of the laser track, as commonly assumed in prior work, but also emerge in heat dissipation regions that evolve with each material's thermophysical properties. Capturing these process condition-dependent shifts, rather than sampling near the laser track alone, proves essential to accurately predicting subsurface and far-field thermal behavior. We additionally conduct an ablation and sensitivity analysis, confirming that integrating domain knowledge is critical to achieving both training stability and superior sample efficiency relative to an uninformative uniform base, which reflects the setting where the residual landscape is unknown a priori. We conclude with a discussion of the limitations of the present study and directions for future work.

In summary, this work contributes a theoretical foundation for collocation sampling in PINNs from a generalization standpoint, a practical adaptive sampling framework that scales across process conditions via amortized generative modeling, and an integration of domain knowledge that ensures stable and efficient training. Together, these contributions move parametric PINNs in metal AM closer to strong generalization and rapid inference across diverse process conditions.

\section{Proposed Method}\label{sec:method}

\subsection{Nomenclature}
We leave the nomenclature in Appendix~\ref{app:nomenclature}.

\subsection{Problem Setup}\label{sec:setup}

Let $\Omega_{x} \subset \mathbb{R}^{d_{x}}$ be the bounded spatial domain representing the part build, and $[0, t_{\mathrm{end}}] \subset \mathbb{R}$ be the time interval over which the AM process is simulated. The spatiotemporal coordinates thus lie in \( (x, t)\in \Omega := \Omega_{x} \times [0, t_{\mathrm{end}}]\), so that $\Omega \subset \mathbb{R}^{d}$ with $d = d_{x} + 1$. Let $\mathcal{P} \subset \mathbb{R}^m$ be a compact set of PDE parameters $\mu \in \mathcal{P}$ within known bounds. In general, $\mu$ can represent any process condition, including, for instance, process parameters (e.g., laser power, scanning speed) and thermophysical properties of the feedstock material (e.g., density, heat capacity, thermal conductivity). Note that while our theoretical analysis applies to general process conditions, our experiments focus specifically on variations in material properties for benchmarking against previous works.

Given the governing PDE defined over $\Omega\times\mathcal{P}$, parametric PINNs employ a deep neural network $f_\theta:\Omega\times\mathcal{P}\to\mathbb{R}$ with learnable parameters $\theta$ to predict the temperature field $T_\theta(x, t; \mu)\in\mathbb{R}$. Training is performed by minimizing a loss function based on the physics-based residuals. Here, the PDE residual is given by \(r_{\theta,\mathrm{pde}}(x,t;\mu) = \mathcal{N}[T_\theta](x,t;\mu)\), 
where $\mathcal{N}[\cdot]$ is the governing nonlinear differential operator. Combined with the corresponding BC and IC residuals, $r_{\theta,\mathrm{bc}}$ and $r_{\theta,\mathrm{ic}}$, alongside their balancing hyperparameters $\lambda_{\mathrm{pde}}, \lambda_{\mathrm{bc}}$, and $\lambda_{\mathrm{ic}}$, the total squared residual is formulated as:
\begin{equation}\label{eq:total_residual}
r_\theta(x,t;\mu)^2 := \lambda_{\mathrm{pde}} r_{\theta,\mathrm{pde}}^2 + \lambda_{\mathrm{bc}} r_{\theta,\mathrm{bc}}^2 + \lambda_{\mathrm{ic}} r_{\theta,\mathrm{ic}}^2
\end{equation}
In this study, the governing PDE corresponds to the 3D transient heat diffusion equation with Fourier conduction, coupled with convective, radiative, and Gaussian laser-flux BCs, and an ambient temperature IC (the detailed formulations are in Appendix~\ref{app:governing_eqs}).

Our goal is to achieve generalization across the entire spatiotemporal domain and PDE parameter space $\Omega\times\mathcal{P}$. We formalize this objective within the standard empirical risk minimization framework by defining two key measures: the population loss $\mathcal{L}(\theta)$ and the empirical loss $\widehat{\mathcal{L}}_N(\theta;q)$. First, the population loss $\mathcal{L}(\theta)$ over $\Omega\times\mathcal{P}$ is defined as:
\begin{equation}\label{eq:Lpop}
\mathcal{L}(\theta)
:= \int_{\Omega\times\mathcal{P}}
r_\theta(x,t;\mu)^2\,
p(x,t,\mu)\,dx\,dt\,d\mu
\end{equation}
where $p(x,t,\mu)$ is assumed to be the uniform distribution over $\Omega\times\mathcal{P}$. However, in practice, $\mathcal{L}(\theta)$ cannot be feasibly computed due to high-dimensional integration. Training therefore relies on Monte Carlo estimation with collocation points $\{(x_i, t_i,\mu_i)\}_{i=1}^N$ drawn i.i.d. from a proposal density $q(x,t,\mu)$. Here, note that we assume both time $t$ and PDE parameters $\mu$ are independently drawn from fixed uniform distributions across both training and test evaluations. Consequently, the joint proposal density factorizes as $q(x, t, \mu) = q(x \mid t, \mu) \, p(t) \, p(\mu)$. Using importance sampling, the empirical loss estimator $\widehat{\mathcal{L}}_N(\theta;q)$ is then constructed as:
\begin{equation}\label{eq:Lemp_estimator}
\widehat{\mathcal{L}}_N(\theta;q) := 
\frac{1}{N}\sum_{i=1}^N \frac{p(x_i,t_i,\mu_i)}{q(x_i,t_i,\mu_i)}\, r_\theta(x_i,t_i;\mu_i)^2 
= \frac{1}{N}\sum_{i=1}^N \frac{p(x_i)}{q(x_i \mid t_i, \mu_i)}\, r_\theta(x_i,t_i;\mu_i)^2
\end{equation}

Throughout this study, we denote $q$ as the general spatial proposal. This includes the conventional static baseline proposals $q(x\mid t)$, which depend on $t$ to follow the time-dependent laser track but remain independent of $\mu$, and our proposed adaptive proposal $q_\theta(x\mid t, \mu)$, which dynamically adapts to both $t$ and $\mu$ for a given $\theta$. Since the target distributions for $t$ and $\mu$ are assumed to be fixed uniform across both training and test settings, we aim to design the adaptive spatial conditional distribution $q_\theta(x\mid t, \mu)$ to minimize the population loss $\mathcal{L}(\theta)$ and thereby achieve effective generalization across unseen process conditions.

\subsection{Theoretical analysis of collocation sampling and generalization}\label{sec:errordecomp}
In this subsection, we provide a theoretical analysis of how the choice of proposal $q$ affects the excess generalization error. We first show that this error decomposes into an optimization error and a statistical error, with the latter bounded by the variance of the empirical loss estimator $\widehat{\mathcal{L}}_N(\theta;q)$.
To lower this bound and thereby minimize the generalization error, we prove that the adaptive proposal $q_\theta(\cdot\mid t, \mu)$ is strictly superior to conventional $\mu$-independent proposals $q(\cdot\mid t)$. Finally, we show that the optimal adaptive proposal is proportional to the squared residual. Figure~\ref{fig:section2_2} provides a visual overview of these theoretical results.

\begin{figure}[th]
    \centering
    \includegraphics[width=1\textwidth]{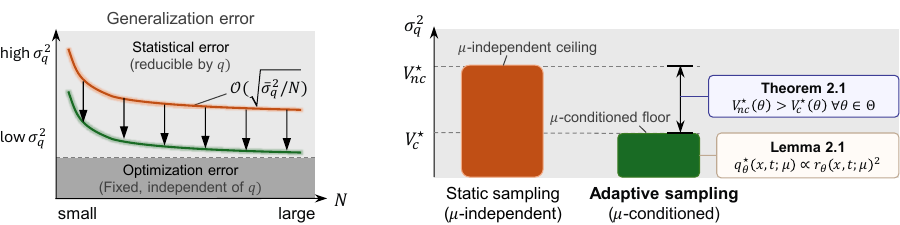} 
\caption{Visual overview of Section~\ref{sec:errordecomp}.
\textbf{Left:} Proposition~\ref{prop:genbound} decomposes the generalization error into an optimization error, which is independent of the proposal $q$, and a statistical error that scales as $\mathcal{O}(\sqrt{\bar{\sigma}^2_q / N})$. This implies that reducing the variance of the empirical loss estimator, $\bar{\sigma}^2_q$, is a more effective lever for minimizing the statistical error, and thereby the overall generalization error, than merely increasing the total sample size $N$.
\textbf{Right:} To reduce $\bar{\sigma}^2_q$, Theorem~\ref{thm:strict-sep} states that conventional $\mu$-independent static sampling is strictly inferior to adaptive sampling, as its infimum variance $V^\star_{\mathrm{nc}}(\theta)$ imposes a hard ceiling, whereas the $\mu$-conditioned adaptive proposal strictly lowers this variance floor. Lemma~\ref{lem:optimal-q} further characterizes the optimal adaptive proposal as being proportional to the squared residual.
}
\label{fig:section2_2}
\end{figure}

Let $\theta^\star \in \arg\min_{\theta \in \Theta} \mathcal{L}(\theta)$ be the population loss minimizer over the hypothesis class $\Theta$, $\widehat\theta_N(q) \in \arg\min_{\theta \in \Theta} \widehat{\mathcal{L}}_N(\theta;q)$ the empirical loss minimizer under proposal $q$, and $\widetilde\theta_N$ the parameter returned after training.
The excess generalization error, the deviation of the population loss of the trained parameter $\mathcal{L}(\widetilde\theta_N)$ from the minimal population loss $\mathcal{L}(\theta^\star)$, admits the standard decomposition \citep{shalev2014understanding}:
\begin{equation}\label{eq:uniform-dev}
\underbrace{\mathcal{L}(\widetilde\theta_N) - \mathcal{L}(\theta^\star)}_{\text{excess generalization error}} \;\leq\; \underbrace{\big[\widehat{\mathcal{L}}_N(\widetilde\theta_N; q) - \widehat{\mathcal{L}}_N(\widehat\theta_N(q); q)\big]}_{\text{optimization error}} \;+\; 2\,\sup_{\theta \in \Theta} \underbrace{\big|\widehat{\mathcal{L}}_N(\theta; q) - \mathcal{L}(\theta)\big|}_{\text{statistical error}}
\end{equation}
The detailed derivation of this decomposition is provided in Appendix~\ref{prf:excgen}. In other words, the excess generalization error is comprised of an optimization error, which is controlled by optimization techniques (e.g., model architecture, optimizer) and, more importantly, does not depend on the proposal $q$, and a statistical error, which depends directly on $q$ and the sample size of collocation points $N$. Thus, reducing this statistical error toward zero is key to achieve true generalization. The statistical error in turn admits the following variance-dependent generalization bound:
\begin{tcolorbox}[colback=teal!5,colframe=teal!50!black,boxrule=0.7pt,arc=2pt,left=6pt,right=6pt,top=4pt,bottom=4pt,breakable]
\begin{proposition}[Variance-dependent generalization bound]\label{prop:genbound}
With probability at least $1 - \delta$, the supremum of the statistical error is bounded as:
\begin{equation}\label{eq:Bernstein}
\sup_{\theta \in \Theta}\underbrace{
\big|\widehat{\mathcal{L}}_N(\theta;q) - \mathcal{L}(\theta)\big|}_{\text{statistical error}}
\;\leq\; C_1\,\sqrt{\frac{\bar\sigma_q^2 \,\mathcal{C}(\Theta,\delta)}{N}}
\;+\; C_2\frac{M_qR^2 \,\mathcal{C}(\Theta,\delta)}{N}
\end{equation}
where $\bar\sigma_q^2 := \sup_{\theta \in \Theta} \sigma_q^2(\theta)$ with $\sigma_q^2(\theta) := N\,\mathrm{Var}_q\big[\widehat{\mathcal{L}}_N(\theta;q)\big]$ denoting the per-sample variance of $\widehat{\mathcal{L}}_N(\theta;q)$.
\end{proposition}
\end{tcolorbox}
\noindent The proof is provided in Appendix~\ref{prf:genbound}. Proposition~\ref{prop:genbound} delivers the central motivation of this paper: for a conventional static proposal $q(\cdot\mid t)$, $\bar\sigma_q^2$ remains a fixed constant that cannot be changed. Moreover, merely increasing $N$ is not computationally efficient as the leading term on the right hand side scales as $O(\sqrt{\bar\sigma_q^2/N})$. For instance, achieving a tenfold reduction in error requires a hundredfold increase in $N$. These observations justify the need for an adaptive sampling proposal $q_\theta(\cdot\mid t, \mu)$, which allows $\sigma_q^2$ to be directly minimized, thereby reducing the statistical error. Therefore, instead of inefficiently increasing $N$, an adaptive proposal enables us to effectively minimize the generalization error by directly targeting the constant $\bar\sigma_q^2$ in \eqref{eq:Bernstein}, whereas the conventional methods cannot directly reduce $\bar\sigma_q^2$.

In this paper, we further provide the theoretical results showing that the conventional static proposal, even the optimal case, is always strictly less effective than the adaptive sampling proposal:

\begin{tcolorbox}[colback=blue!4,colframe=blue!45!black,boxrule=0.7pt,
  arc=2pt,left=6pt,right=6pt,top=4pt,bottom=4pt,
  breakable]
\begin{theorem}\label{thm:strict-sep}
Let $\mathcal{Q}_{\mathrm{c}} = \{q(x\mid t, \mu)\}$ and $\mathcal{Q}_{\mathrm{nc}} = \{q(x\mid t)\}$ denote the families of conditional and unconditional ($\mu$-independent) proposal distributions, respectively. Define the optimal variance bounds as $V_{\mathrm{c}}^\star(\theta) := \inf_{q \in \mathcal{Q}_{\mathrm{c}}} \sigma_q^2(\theta)$ and $V_{\mathrm{nc}}^\star(\theta) := \inf_{q \in \mathcal{Q}_{\mathrm{nc}}} \sigma_q^2(\theta)$. Then, for every $\theta \in \Theta$,
\begin{equation}\label{eq:V-defs}
V_{\mathrm{c}}^\star(\theta) \;\leq\; V_{\mathrm{nc}}^\star(\theta).
\end{equation}
\end{theorem}
\end{tcolorbox}

\begin{tcolorbox}[colback=gray!5,colframe=gray!55!black,boxrule=0.7pt,arc=2pt]
\begin{remark}\label{rem:strict-sep}
Equality $V_{\mathrm{c}}^\star(\theta) = V_{\mathrm{nc}}^\star(\theta)$ holds if and only if the squared residual separates as $r_\theta(x,t;\mu)^2 = a_t(x)\, c_t(\mu)$, meaning its normalized spatial shape is independent of $\mu$.
\end{remark}
\end{tcolorbox}

\noindent The proofs are provided in Appendix~\ref{prf:strict-sep}. In other words, Theorem~\ref{thm:strict-sep} and Remark~\ref{rem:strict-sep} establish the superiority of conditional sampling over conventional static strategies under theoretical guarantee. In the conventional static strategies, their independence from $\mu$ imposes a fundamental performance ceiling $V_{\mathrm{nc}}^\star$. Intuitively, this ceiling is unavoidable because process conditions significantly alter the residual landscapes rather than merely scaling their overall magnitude. Since no single $\mu$-independent spatial profile can simultaneously capture these parameter-dependent shape variations, the residual landscape is inherently non-separable, guaranteeing that $V_{\mathrm{c}}^\star(\theta) < V_{\mathrm{nc}}^\star(\theta)$ strictly holds in metal AM thermal modeling. By conditioning the proposal directly on $\mu$, our method systematically breaks this static variance barrier and approaches $V_{\mathrm{c}}^\star$, enabling robust generalization across diverse process regimes.

Finally, we now characterize the optimal case of adaptive proposal that attains $V_{\mathrm{c}}^\star(\theta)$ explicitly. Within the parameter-conditioned family $\mathcal{Q}_{\mathrm{c}}$, the variance functional $\sigma_q^2(\theta)$ yields a unique minimizer:
\begin{tcolorbox}[colback=green!3,colframe=green!25!gray,boxrule=0.7pt,
arc=2pt,left=6pt,right=6pt,top=4pt,bottom=4pt,
breakable]
\begin{lemma}[Variance-optimal adaptive proposal]\label{lem:optimal-q}
Fix $\theta \in \Theta$ with $\mathcal{L}(\theta) > 0$. Within $\mathcal{Q}_{\mathrm{c}}$, the spatial conditional that minimizes the variance functional $\sigma_q^2(\theta)$ is the one that places spatial mass in proportion to the squared residual under each process condition and each instant:
\begin{equation}\label{eq:q-star}
q^\star_\theta(x \mid t, \mu) \;\propto\; r_\theta(x,t;\mu)^2
\end{equation}
\end{lemma}
\end{tcolorbox}
\noindent The proof is directly from the equality condition of the Cauchy--Schwarz inequality applied in Case 1 of the proof of Theorem~\ref{thm:strict-sep} (Appendix~\ref{prf:strict-sep}).

To sum up, minimizing the generalization error is strongly tied to reducing the statistical error, and the effective lever is the variance of the empirical loss estimator rather than merely expanding the sample size $N$. In this regard, conventional static, $\mu$-independent proposals $q(\cdot\mid t)$ are strictly worse than $\mu$-conditional proposals $q(\cdot\mid t,\mu)$, since the minimal variance of the latter is strictly lower. The $\mu$-conditional proposal attaining that minimum is proportional to the squared residual under each process condition and instant, which is precisely an adaptive sampling proposal.

\subsection{Domain-Informed Adaptive Sampling via Amortized CFM}\label{sec:cfm}

In the previous subsection, we demonstrated that the $\mu$-conditional adaptive proposal $q_\theta(x\mid t,\mu)$ is superior to a conventional static proposal $q(x\mid t)$ and that its optimal form $q^\star_\theta(x\mid t,\mu)$ is proportional to the squared residual $r_\theta(x,t;\mu)^2$. Nevertheless, directly implementing $q^\star_\theta$ in practice faces the following challenges. First, calculating the exact normalization constant for $q^\star_\theta$ is computationally intractable due to the high-dimensional integration of $r_\theta(x,t;\mu)^2$. Second, drawing samples directly from this non-standard distribution $q^\star_\theta$ across every $t$ and the PDE parameter space $\mathcal{P}$ is non-trivial, requiring expensive sampling procedures such as MCMC. Finally, evaluating the squared residual $r_\theta(x,t;\mu)^2$ at every sampling iteration introduces significant computational overhead, especially in high-dimensional domains like $\Omega\times\mathcal{P}\subseteq\mathbb{R}^{d+m}$ in our case.

To address these challenges, we propose a domain-informed adaptive sampling using Flow Matching (FM). We first use physical priors based on metal AM domain knowledge to warm-start the initial predictor, obtaining a stable and reliable residual distribution $r_{\theta_{\mathrm{src}}}$. We then use Flow Matching to approximate the target residual distribution as a continuous function and condition FM on both time $t$ and process conditions $\mu$, which we refer to as Conditional Flow Matching (CFM), to alleviate the curse of dimensionality. Finally, we blend the domain-informed base with the learned high-residual proposal to handle the exploration--exploitation trade-off, and finally train $f_\theta$ with the refined collocation points.

This approach reduces computational overhead and mitigates the curse of dimensionality. Flow Matching enables continuous density modeling, serving as a scalable alternative to traditional sampling algorithms. By further training the FM model conditioned on both time $t$ and process conditions $\mu$, the sampling process is amortized, allowing direct sample generation from high-residual distributions for each process condition $\mu$ in a single forward pass. This streamlines the training pipeline and provides representative collocation points per process condition, enabling the predictor $f_\theta$ to minimize generalization error across unseen process conditions--our primary objective.

\subsubsection{Flow matching background}

We briefly review Flow Matching (FM), the deep generative backbone used in this study. Flow Matching \citep{lipman2022flow} is a likelihood-free training criterion for continuous normalizing flows designed to learn a time-dependent vector field that transports samples from a tractable source distribution to a desired target distribution. Specifically, let $\{p_\tau\}_{\tau \in [0,1]}$ be a family of probability densities on $\mathbb{R}^d$ interpolating from a source density $p_0$ at $\tau=0$ to a target density $p_1$ at $\tau=1$. A time-dependent vector field $v: [0,1] \times \mathbb{R}^d \to \mathbb{R}^d$ generates this probability path if integrating the velocity field $\dot{z}_\tau = v(\tau, z_\tau)$ starting from $z_0 \sim p_0$ yields $z_\tau \sim p_\tau$ for every $\tau \in [0,1]$. To construct this vector field, Flow Matching regresses a parameterized vector field $v_\phi$ onto the target displacement vectors between paired source and target samples. Given a coupling $(x^{(0)}, x^{(1)}) \sim \gamma$ with marginals $p_0$ and $p_1$, the linear interpolant $x^{(\tau)} = (1-\tau)x^{(0)} + \tau x^{(1)}$ thus yields a constant velocity field $\dot{x}^{(\tau)} = x^{(1)} - x^{(0)}$.

A key advantage of Flow Matching lies in its straightforward optimization trajectory. Training Flow Matching involves a simple mean-squared error (MSE) regression while other generative models are prone to exhibit optimization challenges (e.g., GANs) \citep{arjovsky2017towards}. Therefore, it enables rapid conditional sample generation while maintaining training stability across high-dimensional parameter spaces.

\subsubsection{Training procedure}

\begin{figure}[tp]
    \centering
    \includegraphics[width=1\textwidth]{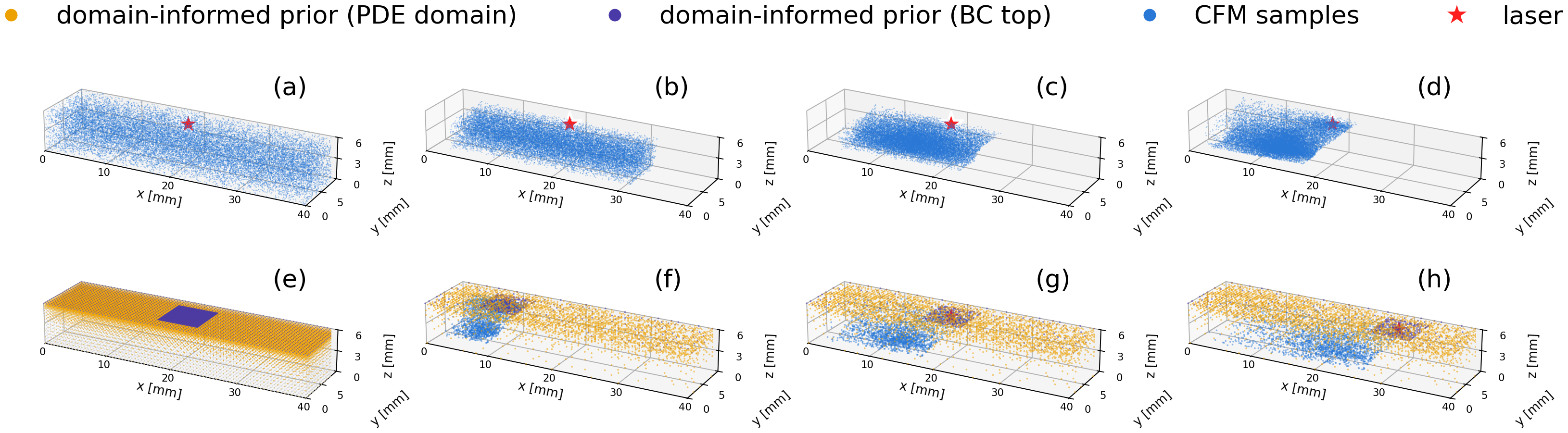} 
    \caption{Domain-informed adaptive collocation via CFM. \textbf{Top} (a--d): CFM transport of uniform base samples ($\tau=0$) into the adaptive proposal $q^\star_{\theta_{\mathrm{src}}}(\cdot\mid t, \mu)$ ($\tau=1$), which concentrates in high-residual regions. \textbf{Bottom}: (e) Domain-informed base $b(\cdot\mid t)$; (f--h) Training mixture $q_\alpha^b$ at $t=0.5, 1.5$, and $2.5\,\text{s}$, where the base remains fixed while the CFM component tracks high-residual regions (visualized for Ti-6Al-4V).}
\label{fig:section2_4}
\end{figure}

Our proposed domain-informed adaptive sampling framework adopts a two-stage training scheme. Stage~1 trains a CFM generator $g_\phi$ to amortize the variance-optimal proposal $q^\star_{\theta_{\mathrm{src}}}(x\mid t,\mu)$, while Stage~2 trains the final predictor $f_\theta$ using a domain-informed mixture proposal that balances exploration and exploitation. The overall training procedure is summarized in Algorithm~\ref{alg:main}.

\begin{algorithm}[t]
\caption{Domain-Informed Adaptive Sampling via Conditional Flow Matching}
\label{alg:main}
\renewcommand{\baselinestretch}{1.0}\selectfont
\begin{algorithmic}[1]
\Require domain-informed base $b(\cdot\mid t)$; mixture weight $\alpha$;
number of stored conditions $J$; pool size $G$; resample size $M_{\mathrm{res}}$;
iteration counts $K_{\mathrm{warm}}$, $K_{\mathrm{flow}}$, $K_{\mathrm{PINN}}$;
batch sizes $B$ (CFM) and $N$ (PINN); ODE steps $N_{\mathrm{ODE}}$.
\Statex
\Statex \textbf{Stage 1: train the amortized generator $g_\phi$}
\State \textbf{(Warm-start)} Train the source predictor $f_{\theta_{\mathrm{src}}}$
  for $K_{\mathrm{warm}}$ steps on the residual loss \eqref{eq:total_residual},
  drawing $t$, $\mu$ uniformly and $x\sim b(\cdot\mid t)$; then freeze
  $\theta_{\mathrm{src}}$.
\State \textbf{(Score)} For each of $J$ uniformly drawn conditions $\mu_j$:
  draw a pool $\mathcal{D}_j=\{(\tilde{x}_k,\tilde t_k)\}_{k=1}^{G}$
  uniformly over $\Omega$ and score every element by
  $s_{j,k}\gets r_{\theta_{\mathrm{src}}}(\tilde{x}_k,\tilde t_k;\mu_j)^2$.
\For{$K_{\mathrm{flow}}$ steps}
  \Comment{CFM training}
  \State Pick $B$ conditions among the $J$ stored ones; for each, resample
    one flow target $(x^{(1)},t^\dagger)$ from $\mathcal{D}_j$ via
    \eqref{eq:resample} and pair it with a source draw $x^{(0)}\sim p_0$.
  \State Update $\phi$ by one gradient step on
    $\mathcal{L}_{\mathrm{CFM}}(\phi)$ in \eqref{eq:cfm-loss-impl}.
\EndFor
\Statex
\Statex \textbf{Stage 2: train the final predictor $f_\theta$ with the frozen generator}
\State Initialize $f_\theta$ from scratch.
\For{$K_{\mathrm{PINN}}$ steps}
  \State Draw $t_i$, $\mu_i$ uniformly and
    $x_i\sim q_\alpha^{b}(\cdot\mid t_i,\mu_i)$ from the mixture
    \eqref{eq:mixture} for $i=1,\dots,N$; samples from the $g_\phi$
    component are obtained by integrating $v_\phi$ with
    $N_{\mathrm{ODE}}$ forward-Euler steps from $x^{(0)}\sim p_0$.
  \State Update $\theta$ by one gradient step on the residual loss
    $\tfrac{1}{N}\sum_{i=1}^{N} r_\theta(x_i,t_i;\mu_i)^2$.
\EndFor
\State \Return the generator $g_\phi$ and the predictor $f_\theta$.
\end{algorithmic}
\end{algorithm}

\paragraph{Stage 1: Warm-start \& Amortized Generator Training}

Let $b(x\mid t)$ denote a $\mu$-independent, domain-informed base density
(visualized in Figure~\ref{fig:section2_4}(e)).
Collocation points are initially drawn from $b$ to warm-start the predictor $f_{\theta_{\mathrm{src}}}$, until it yields physics-consistent solutions with a reliable residual distribution $r_{\theta_{\mathrm{src}}}$.

We now apply CFM to our adaptive sampling setting. Here, our goal is to learn a velocity field $v_\phi$ transporting a simple source density $p_0$ on $\Omega_{x}$ to the variance-optimal target $q^\star_{\theta_{\mathrm{src}}}(\cdot\mid t,\mu)$. We let $p_0$ be uniform on $\Omega_{x}$, which is convenient to sample while able to reach any high-residual region. Note that the flow-matching time $\tau \in [0,1]$ is entirely distinct from the physical time $t$, which enters only as a conditioning input held fixed along the path. In practice, collocation points for the PDE, BC, and IC residuals are drawn on different subdomains; we therefore additionally condition $v_\phi$ on a residual-group identifier $\ell\in\{\mathrm{pde},\mathrm{bc},\mathrm{ic}\}$, which we omit from the notation below for brevity.

To construct the targets, we freeze the source parameters $\theta_{\mathrm{src}}$ and work one condition at a time. For each of $J$ uniformly sampled process conditions from $\mu_j\sim\mathcal{P}$, we draw a pool $\mathcal{D}_j = \{(\tilde{x}_{j,k}, \tilde t_{j,k})\}_{k=1}^{G}$ uniformly over $\Omega$, score it by $s_{j,k} = r_{\theta_{\mathrm{src}}}(\tilde{x}_{j,k}, \tilde t_{j,k}; \mu_j)^2$, and resample $M_{\mathrm{res}}$ elements:
\begin{equation}\label{eq:resample}
\{(x_{j,i}^\dagger,\, t_{j,i}^\dagger)\}_{i=1}^{M_{\mathrm{res}}}
\overset{\text{i.i.d.}}{\sim}
\mathrm{Multinomial}\!\big(\mathcal{D}_j;\; w_{j,k} \propto s_{j,k}\big),
\qquad j = 1,\dots,J.
\end{equation}
Uniform candidates make the weights proportional to the residual itself, with no proposal density to estimate. We then split each draw, conditioning on $(\mu_j, t_{j,i}^\dagger)$ and taking $x_{j,i}^\dagger$ as the flow target; because the weights are normalized within a condition, this recovers exactly $q^\star_{\theta_{\mathrm{src}}}(\cdot\mid t,\mu)$ of Lemma~\ref{lem:optimal-q}. Per-condition normalization also fixes every $\mu_j$ at $M_{\mathrm{res}}$ targets, so the parameter marginal is uniform by construction.

Pairing each target with an independent source draw $x_i^{(0)} \sim p_0$, we train $v_\phi$ via stochastic gradient descent on
\begin{equation}\label{eq:cfm-loss-impl}
\mathcal{L}_{\mathrm{CFM}}(\phi) =
\mathbb{E}_{\tau \sim U[0,1]}\,
\mathbb{E}_{(t,\, \mu,\, x^{(0)},\, x^{(1)})}
\!\Big[\big\|v_\phi\!\big(\tau,\, x^{(\tau)},\, t,\, \mu\big)
- \big(x^{(1)} - x^{(0)}\big)\big\|^2\Big],
\end{equation}
where $x^{(\tau)} := (1-\tau)x^{(0)} + \tau\,x^{(1)}$ is the linear interpolant and $x^{(1)} = x_i^\dagger$. The trained $v_\phi$ then defines the
generator $g_\phi(\cdot\mid t,\mu) \approx q^\star_{\theta_{\mathrm{src}}}(\cdot\mid t,\mu)$ as the law of
$x^{(1)}$ obtained by integrating
$\tfrac{d}{d\tau}x^{(\tau)} = v_\phi(\tau, x^{(\tau)}, t, \mu)$ from
$x^{(0)} \sim p_0$ to $\tau = 1$, which we approximate with the forward Euler
scheme using $N_{\mathrm{ODE}} = 10$.

\paragraph{Stage 2: Predictor Training via Domain-Informed Mixture}
In Stage 2, we freeze the trained CFM generator $g_\phi$ and initialize the final predictor $f_\theta$ from scratch. 
In contrast to Stage 1, the collocation points $x$ is drawn from both the learned proposal $g_\phi$ and the base $b$ with the proposed domain-informed mixture proposal $q_\alpha^{b}$:
\begin{equation}\label{eq:mixture}
q_\alpha^{b}(x\mid t,\mu) := (1-\alpha)\,g_\phi(x\mid t,\mu) + \alpha\,b(x\mid t)
\end{equation}
where the mixture weight $\alpha$ controls the exploration--exploitation trade-off: smaller $\alpha$ more aggressively exploits the learned high-residual regions, while larger $\alpha$ retains broader coverage that defends against estimation error in $r_\theta$. The predictor parameter $\theta$ are then updated on the residual loss evaluated at collocation points drawn from $q_\alpha^{b}$. This prevents the risk that exclusive sampling from the learned high-residual proposal $g_\phi$ neglects underexplored regions where $f_\theta$ initially underestimates the residual.

\section{Experiments}\label{sec:experiments}

\paragraph{System Description}
To evaluate the proposed method, we adopt a widely recognized numerical benchmark \citep{liao2023hybrid, faegh2025physics, lee2026material} based on a bare-plate laser scanning example representative of metal AM processes. In this benchmark, a moving laser heat source scans the surface of a metal substrate, generating highly localized, transient, and rapid thermal gradients governed by laser heating, conductive heat transfer within the substrate, and convective/radiative heat losses at the boundaries. Representative snapshots of the resulting temperature field are illustrated in Figure~\ref{fig:ti64_snapshots}. To ensure a fair comparison, all simulation setups and process parameters follow the previous research \citep{lee2026material}.

\begin{figure}[htb]
\centering
\includegraphics[width=\linewidth]{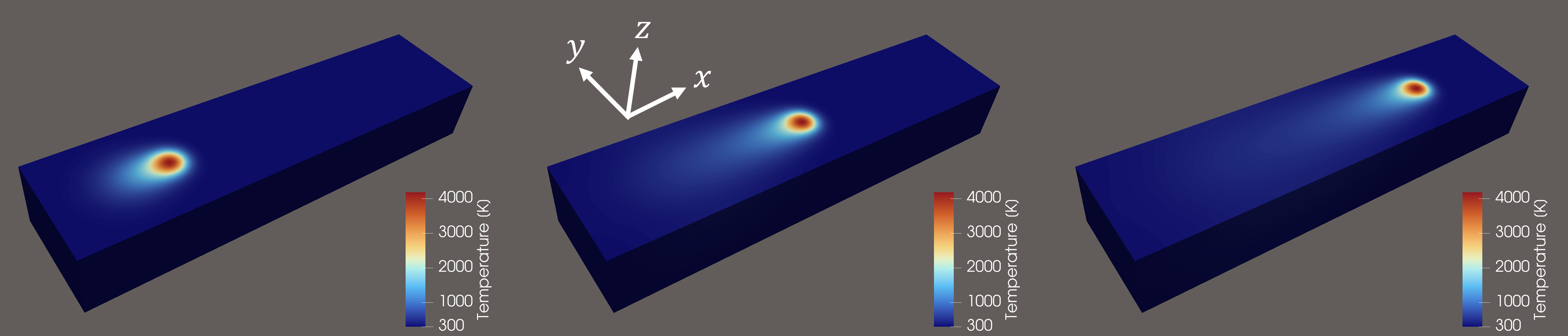}
\caption{
Ground-truth temperature field snapshots at $t = 0.5$, $1.5$, and $2.5\,\mathrm{s}$ obtained via high-fidelity \texttt{JAX-AM} simulation for the bare-plate scanning benchmark (visualized using Ti--6Al--4V). The thermal evolution captures the localized moving laser heat source and subsequent heat diffusion.
}
\label{fig:ti64_snapshots}
\end{figure}

\paragraph{Process conditions $\mu$.}
Following the experimental design in \citet{lee2026material}, we consider five materials widely used in metal AM as unseen process conditions $\mu$ to assess the generalizability of the proposed framework. It is noteworthy that variations in each material's thermophysical properties substantially alter its thermal behavior, inducing significant shifts in the underlying residual landscapes. We define the continuous process condition space $\mathcal{P}$ over the characteristic ranges of density $\rho_c$, heat capacity $C_p$, and thermal conductivity $k_c$:
\begin{equation*}
\mathcal{P} := [\rho_{c,\text{min}}, \rho_{c,\text{max}}] \times [C_{p,\text{min}}, C_{p,\text{max}}] \times [k_{c,\text{min}}, k_{c,\text{max}}]
\end{equation*}
where the boundaries are set to $\rho_c \in [3000, 10000]\,\mathrm{kg/m}^3$, $C_p \in [300, 1000]\,\mathrm{J/(kg\cdot K)}$, and $k_c \in [3, 100]\,\mathrm{W/(m\cdot K)}$, identically applied at both training and testing phases as formulated in Section~\ref{sec:setup}. We evaluate five representative materials with properties denoted as $(\rho_c, C_p, k_c)$:

\begin{itemize}
    \item \textbf{Ti--6Al--4V} $(4430, 560, 6.7)$: Standard aerospace alloy characterized by a high strength-to-weight ratio.
    \item \textbf{Inconel 718} $(8220, 435, 11.4)$: Nickel-based superalloy featuring high-temperature thermal stability.
    \item \textbf{SS 316L} $(8000, 500, 16.0)$: Austenitic stainless steel with high corrosion resistance.
    \item \textbf{AlSi10Mg} $(2670, 950, 150)$: Lightweight structural alloy whose thermal conductivity $k_c$ exceeds the upper limit of $\mathcal{P}$ by $1.5\times$.
    \item \textbf{Copper} $(8960, 385, 401)$: Highly conductive material with an extreme thermal conductivity $k_c$ exceeding the upper bound of $\mathcal{P}$ by more than $4\times$.
\end{itemize}

Here, the first three materials represent in-distribution (ID) cases, whereas AlSi10Mg and Copper serve as OOD extrapolation cases as their thermal conductivities $k_c$ substantially exceed the training upper bound $k_{c,\text{max}}$. Note that OOD evaluation is known to be notoriously challenging compared to ID cases, allowing us to validate whether the extrapolation capability demonstrated in prior work \citep{lee2026material} remains effective under our adaptive sampling scheme.

\paragraph{Baselines and Evaluation Metric}
We evaluate the proposed method against two PINN baselines.

\begin{itemize}
    \item \textbf{Non-parametric PINN (N-PINN)} \citep{liao2023hybrid}: Conventional PINN formulation where material properties are treated as fixed constants. A separate network must be trained from scratch for each material condition.

    \item \textbf{Advanced Parametric PINN (AP-PINN)} \citep{lee2026material}: State-of-the-art parametric PINN for material-agnostic thermal modeling in metal AM, trained using a static domain-informed collocation strategy $b(x\mid t)$. Our predictor $f_\theta$ adopts this architecture as our backbone without modification, ensuring that any performance gain is entirely attributable to the proposed adaptive sampling strategy.
\end{itemize}

For quantitative comparison, we report the relative $L_2$ error, a standard metric in the PINN literature \citep{raissi2019physics}. This metric measures the normalized discrepancy between the predicted temperature field $T_\theta$ and the ground-truth solution $T$ over the uniformly discretized spatiotemporal domain $\Omega$. For a given process condition $\mu$, the relative $L_2$ error is defined as:
\begin{equation*}
\text{Relative } L_2 \text{ Error} (\%) = \frac{\sqrt{\sum_{n=1}^{N_\text{FEM}} \left| T\left(x^{(n)},t^{(n)}; \mu\right) - T_{\theta}\left(x^{(n)},t^{(n)};\mu\right) \right|^2}}{\sqrt{\sum_{n=1}^{N_\text{FEM}} \left| T\left(x^{(n)},t^{(n)}; \mu\right) \right|^2}} \times 100
\end{equation*}
Here, $T$ represents the reference ground-truth temperature field obtained via \texttt{JAX-AM}\footnote{\url{https://github.com/CMSL-HKUST/jax-am}} \citep{xue2023jax}, an open-source library for high-fidelity FEM simulations in metal AM. $N_\text{FEM}$ denotes the total number of nodes in the FEM simulations, which is $N_\text{FEM} = 88{,}939$ in this study.

\paragraph{Implementation Details}
To maintain consistency across all models, baselines are implemented following their original hyperparameters and training configurations. The loss weights are set to $(\lambda_{\mathrm{pde}}, \lambda_{\mathrm{bc}}, \lambda_{\mathrm{ic}}) = (1, 1, 10^{-4})$ following \citet{liao2023hybrid, lee2026material}. Both baselines adopt the same domain-informed static collocation sampling strategy, and we adopt their defined base distribution $b(x \mid t)$ for our method as well. Furthermore, the total collocation point budget $N$ and its per-group allocations ($N_\text{pde}, N_\text{bc}, N_\text{ic}$) strictly match the baseline setup, guaranteeing an identical sampling budget across all methods.

For our sampling scheme, the CFM velocity network $v_\phi$ is parameterized using $3$ hidden layers of width $128$ with SiLU activations, encoding the flow-matching time $\tau$ via a $32$-dimensional sinusoidal embedding. The generator $g_\phi$ is a substantially smaller network relative to our predictor $f_\theta$, which is identical to the predictor of the AP-PINN baseline. The mixture weight is fixed to $\alpha = 0.6$, and its hyperparameter sensitivity is discussed in the following subsection. Note that training the CFM generator takes approximately $2$ minutes on a single GPU and drawing adaptive samples adds less than $1\%$ to the per-run training time relative to training $f_\theta$.

All quantitative results are reported as the mean $\pm$ standard deviation computed across five random seeds. Experiments were conducted on a single NVIDIA RTX 5090 GPU using Python 3.10 and PyTorch	2.9. The complete implementation, including data and code are available at \url{https://github.com/hsleecri/AdaptiveSamplingPINN-AM}.

\newpage
\section{Results and Discussion}\label{sec:results}

\paragraph{Quantitative Evaluation}
Table~\ref{tab:results} summarizes the quantitative comparison between the baselines and our proposed method. (Qualitative temperature fields from our predictions and FEM ground-truth solutions at $t = 0.5\text{ s}, 1.5\text{ s}$, and $2.5\text{ s}$ are provided in Appendix~\ref{app:qualit}.) While the state-of-the-art parametric PINN baseline (AP-PINN) outperforms the non-parametric baseline (N-PINN) across five material spanning both ID and OOD regimes, our proposed method considerably outperforms both baselines. Specifically, our method further reduces the relative $L_2$ error of AP-PINN by an average of $62.1\%$ and up to $70.1\%$, achieving consistent gains across all materials. 

\begin{table}[htbp]
\centering
\caption{Relative $L_2$ error [\%] (mean $\pm$ std over five seeds) across five materials. \emph{Average} denotes the mean across all five materials. The last row reports the error reduction of the proposed method relative to the AP-PINN baseline. Lower is better.}
\label{tab:results}
\resizebox{\linewidth}{!}{
\begin{tabular}{l c ccc cc}
\toprule
 & & \multicolumn{3}{c}{In-distribution (ID)} & \multicolumn{2}{c}{Out-of-distribution (OOD)} \\
\cmidrule(lr){3-5}\cmidrule(lr){6-7}
Method & Average & Ti-6Al-4V & Inconel 718 & SS 316L & AlSi10Mg & Copper \\
 & & $(4430, 560, 6.7)$ & $(8220, 435, 11.4)$ & $(8000, 500, 16.0)$ & $(2670, 950, 150)$ & $(8960, 385, 401)$ \\
\midrule
N-PINN& $4.82$ & $6.19 \pm 1.00$ & $6.09 \pm 0.83$ & $5.94 \pm 0.63$ & $4.25 \pm 0.35$ & $1.63 \pm 0.13$ \\
\midrule
AP-PINN& $3.91$ & $5.80 \pm 0.32$ & $4.85 \pm 0.30$ & $4.85 \pm 0.28$ & $2.86 \pm 0.07$ & $1.18 \pm 0.05$ \\
\textbf{Proposed} & $\mathbf{1.48}$ & $\mathbf{2.55 \pm 0.49}$ & $\mathbf{1.71 \pm 0.36}$ & $\mathbf{1.45 \pm 0.21}$ & $\mathbf{1.19 \pm 0.21}$ & $\mathbf{0.51 \pm 0.10}$ \\
\midrule
\rowcolor{gray!15}
\textbf{Reduction [\%]} & $\mathbf{62.1\%}$ & $\mathbf{56.0\%}$ & $\mathbf{64.7\%}$ & $\mathbf{70.1\%}$ & $\mathbf{58.4\%}$ & $\mathbf{56.8\%}$ \\
\bottomrule
\end{tabular}
}
\end{table}

\begin{figure}[H]
  \centering
  \includegraphics[width=\linewidth]{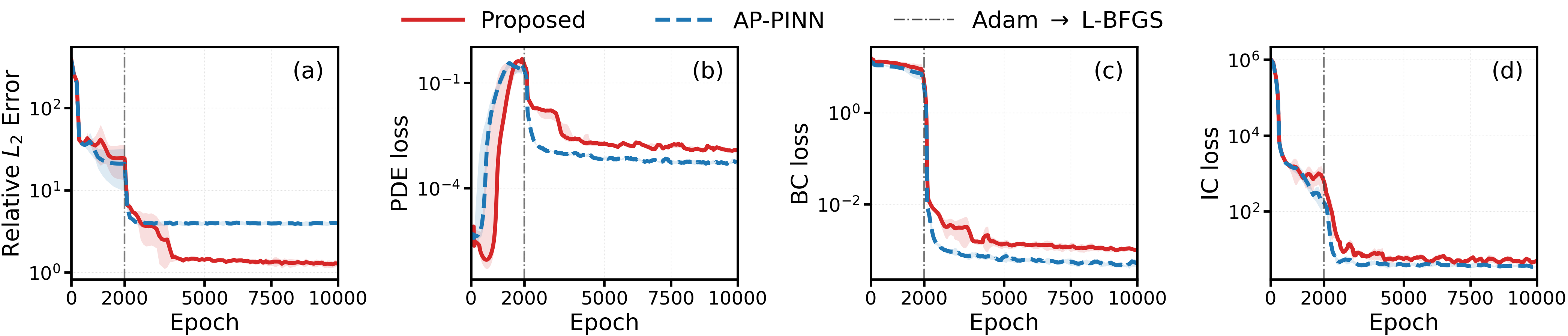}
    \caption{Training dynamics (mean $\pm$ std over five seeds): (a) average relative $L_2$ error on the five materials, and (b--d) the PDE, BC, and IC training losses evaluated on each method's own sampled collocation points. The proposed method reaches a lower relative $L_2$ error despite higher training losses on the three losses, which is a direct consequence of concentrating collocation points in high-residual regions.}
  \label{fig:training_dynamics}
\end{figure}
 
\paragraph{Training Loss Analysis}
To discover where this superiority stems from, Figure~\ref{fig:training_dynamics} illustrates the training dynamics of AP-PINN alongside our proposed method. Interestingly, the PDE, BC, and IC training losses of the proposed method are consistently higher than those of AP-PINN, whereas our relative $L_2$ error is substantially lower. We attribute this phenomenon to the expected signature of adaptive collocation: our adaptive sampling proposal concentrates a higher proportion of points in high-residual regions, naturally leading to a larger residual over the sampled points.

Specifically, after the transition of optimizers from Adam to stochastic L-BFGS (gray dashed line), our model starts to minimize the three losses simultaneously, indicating its convergence toward a good local optimum. During this stage, PDE and BC losses are dramatically reduced, and the relative $L_2$ error rapidly decreases at the same time. This indicates that, despite these higher evaluated residuals, the strategically placed collocation points allow the model to learn the underlying thermodynamics far more effectively across both the spatiotemporal domain and PDE parameter space. Consequently, effective training on points sampled from high-residual regions ultimately yields significantly lower relative $L_2$ errors.

\begin{figure}[htbp]
  \centering
  \includegraphics[width=0.8\linewidth]{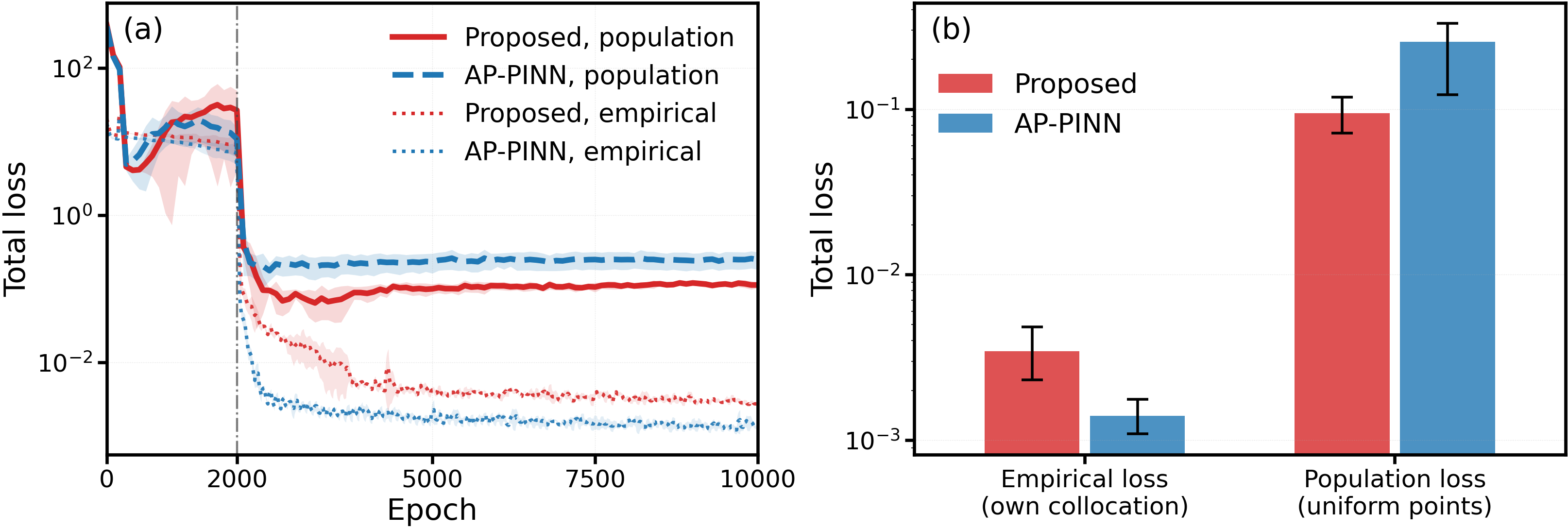}
\caption{Comparison of empirical loss $\widehat{\mathcal{L}}_N(\theta;q)$~\eqref{eq:Lemp_estimator} and population loss $\mathcal{L}(\theta)$~\eqref{eq:Lpop} (mean $\pm$ std over five seeds): (a) training evolution and (b) final values. The proposed method exhibits a higher empirical loss during training, yet consistently achieves a lower population loss by focusing samples on high-residual regions, which leads to more effective learning.}
  \label{fig:total_loss}
\end{figure}

To substantiate this claim, we re-evaluate the training dynamics using the weighted sum of residuals~\eqref{eq:total_residual} computed on collocation points uniformly sampled over $\Omega \times \mathcal{P}$, which serves as an unbiased estimator of the true population loss $\mathcal{L}(\theta)$, rather than relying on the empirical loss estimator $\widehat{\mathcal{L}}_N(\theta;q)$ evaluated under each method's respective proposal distribution (i.e., the learned proposal $q^b_\alpha(x \mid t, \mu)$ for our method and the static domain-informed proposal $b(x \mid t)$ for AP-PINN). As shown in Figure~\ref{fig:total_loss}, the proposed method exhibits a substantially lower estimated population loss compared to AP-PINN. Conversely, the empirical loss evaluated on its own proposal is higher for our method. This confirms that the proposed adaptive proposal concentrates collocation points in high-residual regions during training, ultimately leading to superior generalization performance.


\begin{figure}[htbp]
  \centering
  \begin{subfigure}{\linewidth}
    \centering
    \includegraphics[width=\linewidth]{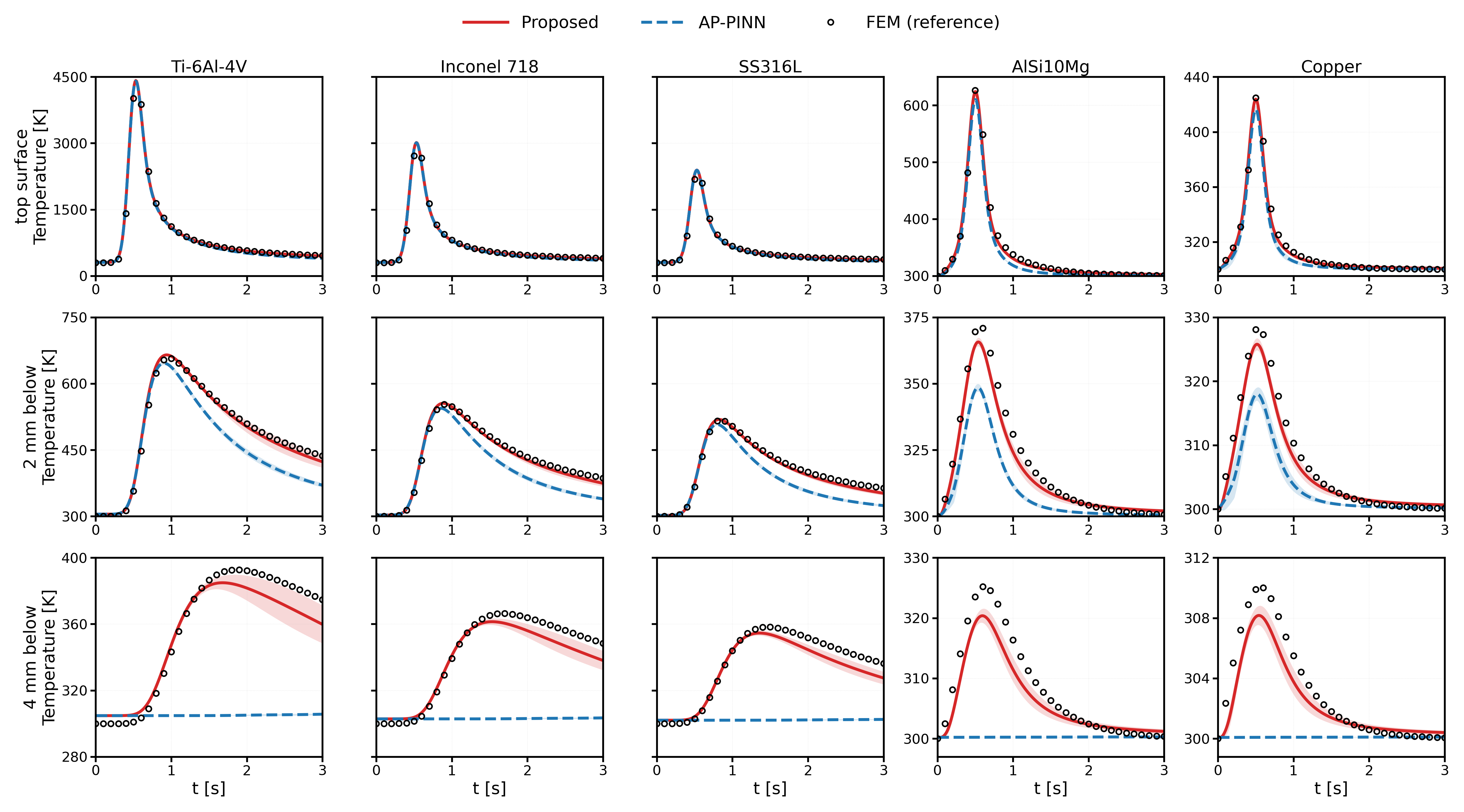}
    \caption{Thermal histories predicted by AP-PINN and ours}
    \label{fig:thermal_history_a}
  \end{subfigure}%
  \\[1ex]
  \begin{subfigure}{\linewidth}
    \centering
    \includegraphics[width=\linewidth]{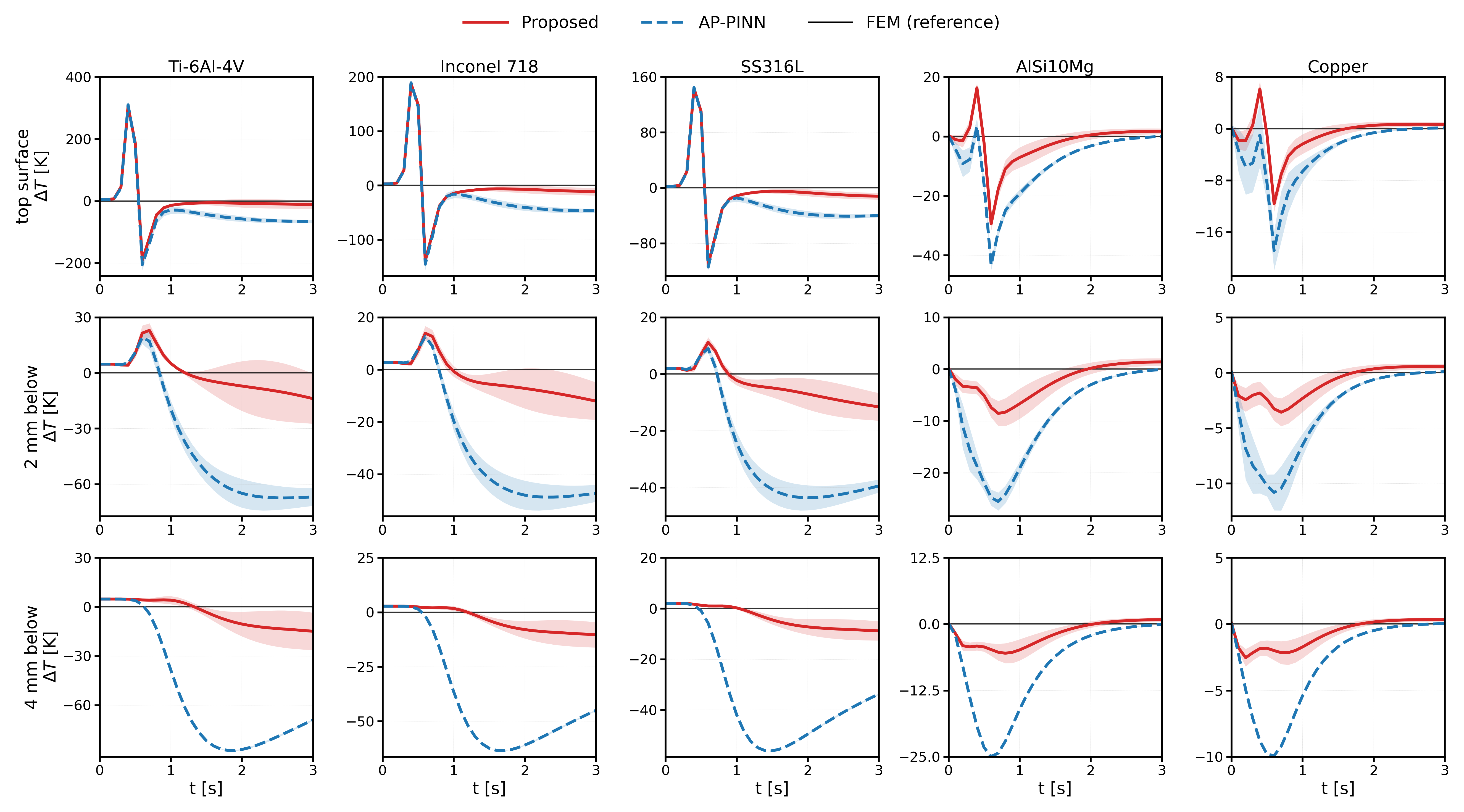}
    \caption{Deviations from FEM by AP-PINN and ours}
    \label{fig:thermal_history_b}
  \end{subfigure}
  \caption{Comparison of thermal histories at increasing depths below the laser scan track across five materials (columns) at the top surface, $2\,\mathrm{mm}$, and $4\,\mathrm{mm}$ depth (rows) (mean $\pm$ std over five seeds): (a) predicted temperature profiles and FEM reference and (b) temperature deviations $\Delta T = T_\theta - T$ relative to the FEM reference. Both methods match well at the top surface under direct laser scanning, but the AP-PINN baseline underperforms significantly in late-time and far-field regions where heat dissipation governs over laser irradiation. Ultimately, AP-PINN misses the subsurface thermal profiles entirely at $4\,\mathrm{mm}$ depth across all materials, whereas the proposed method accurately captures it.}
  \label{fig:thermal_history_combined}
\end{figure}

\paragraph{Thermal History Comparison}
To examine how these reductions in both residuals and relative $L_2$ error translate to temperature field predictions, we compare the thermal histories predicted by AP-PINN and the proposed method against the ground-truth FEM solutions in Figure~\ref{fig:thermal_history_a}, alongside their corresponding temperature deviations in Figure~\ref{fig:thermal_history_b}. This comparison is evaluated at three specific positions below the scan track ($x=10\,\mathrm{mm}$) to capture the distinct thermal behavior of each material: at the top surface ($z=0\,\mathrm{mm}$), $2\,\mathrm{mm}$ below, and $4\,\mathrm{mm}$ below the top surface (which corresponds to $2\text{ mm}$ upward from the bottom surface).

As shown in Figure~\ref{fig:thermal_history_a}, both methods closely match the FEM reference at the top surface. They both successfully capture the rapid transient thermal gradients induced by laser irradiation across all materials, as well as accounting for their distinct temperature magnitudes driven by varying thermophysical properties. At $2\,\mathrm{mm}$ depth, however, deviations begin to emerge as differences in thermophysical properties lead to distinct thermal conduction patterns. With increasing depth, the predictions diverge significantly. At $4\,\mathrm{mm}$ depth, AP-PINN predicts a constant ambient temperature that completely misses the subsurface thermal profile, whereas our method accurately captures both the peak magnitude and profile evolution.

A closer look at Figure~\ref{fig:thermal_history_b} reveals the primary advantages of our approach in capturing heat dissipation regions. For the materials with low thermal conductivity (i.e., Ti-6Al-4V, Inconel 718, and SS 316L), deviations grow more noticeable over time and at greater distances from the laser track.
Conversely, for AlSi10Mg and Copper with much higher thermal conductivity and thus higher thermal diffusivity ($ \frac{k_c}{\rho_c C_p}$), the performance gap appears early during laser irradiation on the top surface due to rapid heat dissipation in these highly conductive materials. As the heat absorbed from the laser irradiation quickly dissipates toward the subsurface due to high thermal diffusivity, the deviations between the two methods are notable at early times rather than at late times, unlike in the in-distribution materials with low thermal conductivity and low thermal diffusivity. Therefore, it can be confirmed that our adaptive collocation strategy successfully captures distinct, material-dependent thermal behaviors, with advantages particularly pronounced in heat-dissipating regions: at late times in the far field for low-conductivity materials, and across both early times and near-to-far fields for high-conductivity materials. Finally, the near-zero standard deviation of AP-PINN's predicted temperature further confirms that this failure stems from a structural limitation in its static collocation.

Despite the superiority of our method especially in the heat dissipation regions, both approaches exhibit notable deviations from the FEM reference under direct laser irradiation, particularly for low-conductivity materials. This discrepancy likely stems in part from the high-frequency nature of these regions, driven by the steep thermal gradients and spikes induced by laser irradiation. This aligns with the well-known spectral bias of neural networks, which is further amplified in PINNs because the residual involves high-order derivatives, making high-frequency regions notoriously difficult to resolve \citep{tancik2020fourier, wang2021eigenvector}. Because our problem formulation assumes a uniform temporal distribution of time steps, the sampling density remains uniform near these transient thermal peaks. Incorporating time steps into our adaptive sampling scheme could alleviate this discrepancy, though at the cost of additional optimization and computational complexity. A more fundamental solution may lie in alternative feature engineering of the spatiotemporal domain, for instance, mapping to the Fourier domain to better resolve high-frequency regions (as in Fourier Neural Operators (FNO) \citep{li2020fourier}), which we leave as a promising direction for future work.

\begin{figure}[htbp]
\begin{subfigure}{\linewidth}
  \includegraphics[width=1\textwidth]{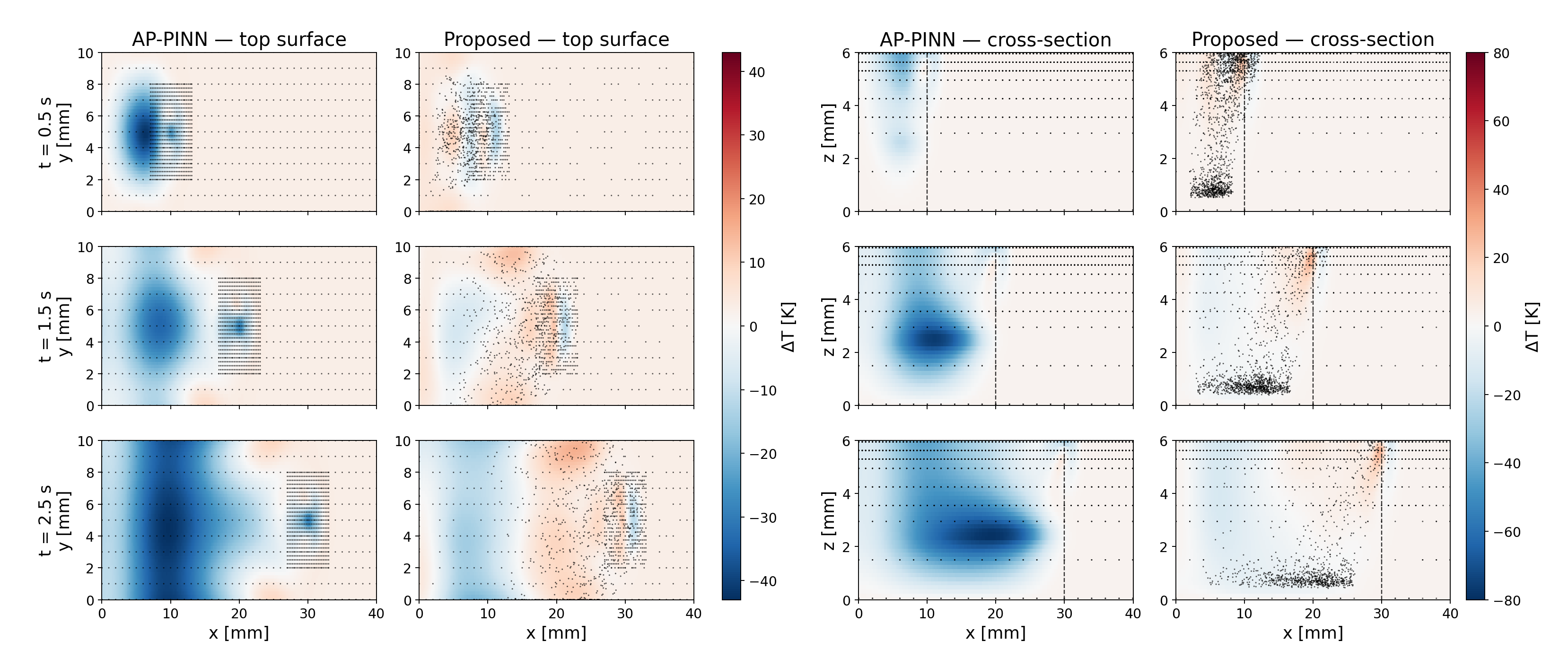}
  \caption{Inconel 718 ($k_c = 11.4$~W\,m$^{-1}$\,K$^{-1}$)}
  \label{fig:dev-in718}
\end{subfigure}
\begin{subfigure}{\linewidth}
  \includegraphics[width=1\textwidth]{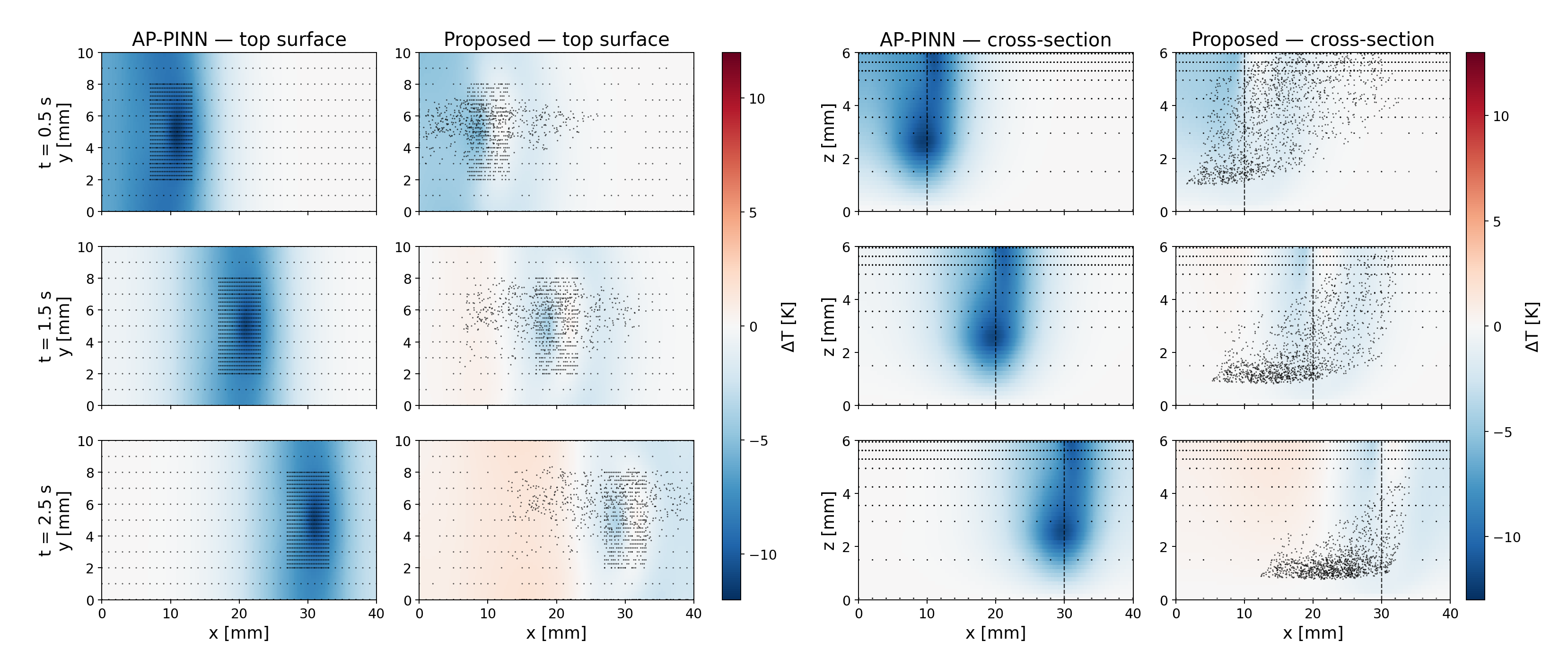}
  \caption{Copper ($k_c = 401$~W\,m$^{-1}$\,K$^{-1}$)}
  \label{fig:dev-cu}
\end{subfigure}
\caption{Temperature deviation $\Delta T = T_\theta - T$ at $t = 0.5$, $1.5$ and $2.5$~s (rows), with each method's collocation points overlaid. Top-surface panels show the top boundary, cross-section panels the domain group; the dashed line marks the laser center. Both the total number of collocation points are identical while AP-PINN draws them from $b(x\mid t) $ whereas the proposed method from the mixture $q_\alpha^{b} $ in \eqref{eq:mixture} with $\alpha = 0.6$. The collocation points drawn in AP-PINN are concentrated nearby the laser irradiation zone and the top surface, which results in significant discrepancy from the FEM reference at heat dissipation regions. In contrast, the proposed method effectively allocates the collocation points on both the laser irradiation zone and the heat dissipation regions at a given time step, resulting in considerable attenuation of the discrepancies shown in the AP-PINN baseline.}
\label{fig:dev-main}
\end{figure}

\paragraph{Qualitative Analysis}
Figure~\ref{fig:dev-main} compares the top-surface and cross-sectional temperature field deviations relative to the ground-truth FEM solutions for both the AP-PINN baseline and our proposed method at $t = 0.5$, $1.5$, and $2.5$\,s. Overlaid on the deviation fields are the collocation points sampled by each strategy: the conventional, manually crafted static sampling and our domain-informed adaptive sampling. Here, Inconel 718 and Copper are selected as representative low- and high-thermal-conductivity materials, respectively. Deviations for the remaining evaluation materials are provided in Appendix~\ref{app:dev}.

Notably, the proposed method substantially alleviates the deviation between the AP-PINN baseline and FEM reference in all panels. AP-PINN exhibits a growing region of under-prediction in heat dissipation regions for both materials and at every time step. For the low-conductivity material Inconel 718, at the early time step ($t=0.5$\,s), when the laser has only just begun irradiating the top surface, the absorbed heat remains concentrated near the melt pool, and the temperature deviations are similarly localized to this region. As the laser moves along the track and time passes, this failure to capture heat dissipation propagates along each material's heat dissipation path: toward the subsurface where heat accumulates for the low-conductivity Inconel 718, and nearly parallel to the laser track for the high-conductivity Copper. These consistent failures highlight a fundamental limitation of conventional, manually crafted static sampling, which concentrates points almost exclusively near the laser track to ensure stable convergence of PINNs in metal AM, leaving substantial room for improvement in capturing the rapid-cooling zones where heat dissipates.

Our domain-informed adaptive proposal addresses this limitation by augmenting the conventional sampling region with a learned proposal that accounts for the distinct thermal behaviors induced by process conditions, dynamically allocating points along the evolving heat dissipation pathways where errors would otherwise accumulate. Because these pathways vary considerably with thermophysical material properties, effective collocation requires more than sampling the melt pool along the laser track alone. By dynamically reallocating points to these fast-cooling, high-gradient regions, our method suppresses the deviations prominent in the baseline to within approximately $\pm 10$\,K across all time steps and all five evaluation materials, with progressively larger gains as heat continues to accumulate and dissipate over time.

\begin{figure}[tbh]
  \centering
  \includegraphics[width=\linewidth]{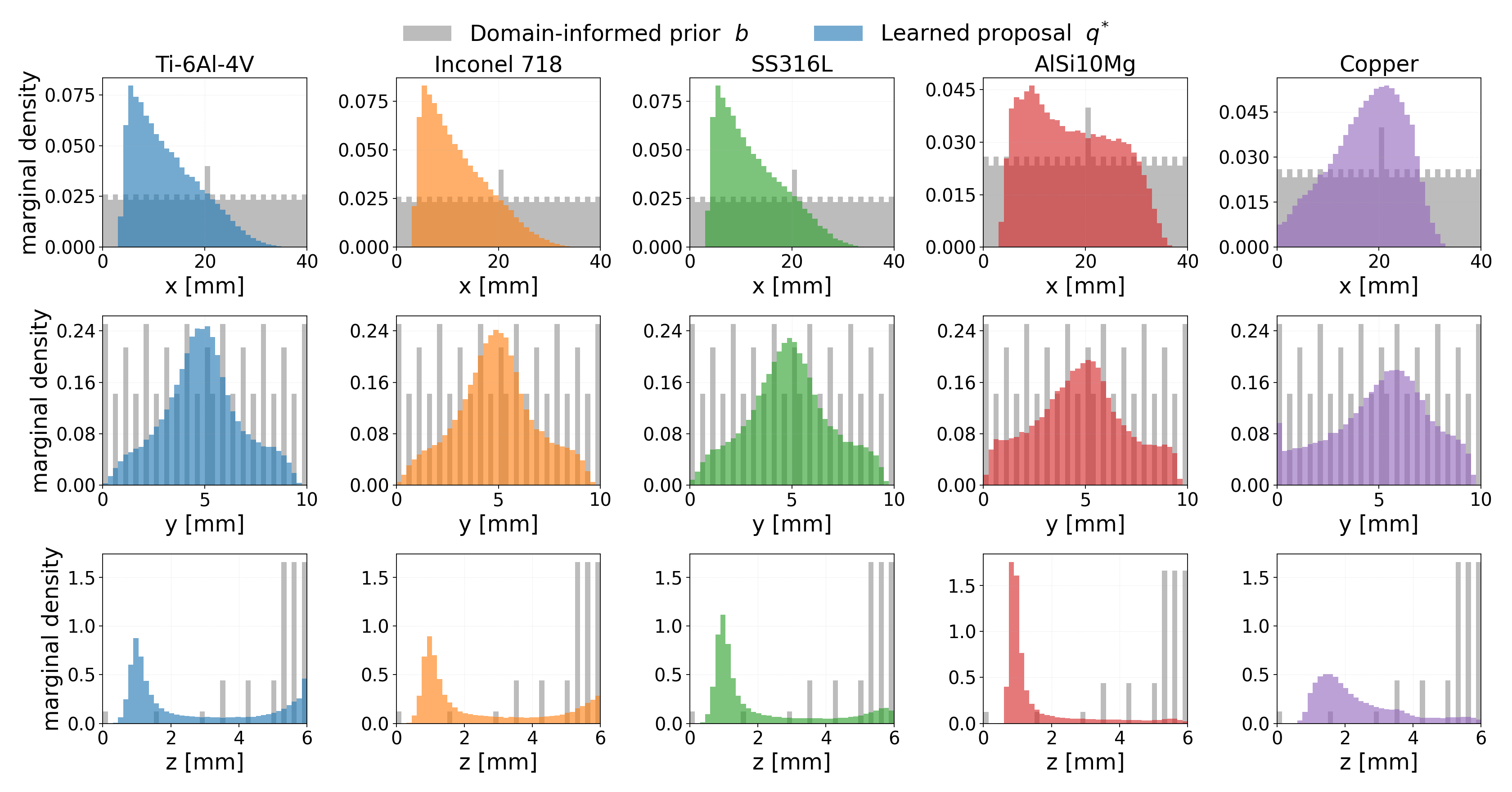}
  \caption{Marginal distributions of collocation points along each spatial axis, averaged over $t$, across the five evaluation materials.}
  \label{fig:marginals_per_material}
\end{figure}

\paragraph{Collocation Sampling Distribution Comparison}
Figure~\ref{fig:marginals_per_material} presents the marginal distributions along each spatial dimension where the conventional approach and our adaptive sampling proposal concentrate collocation points. The conventional, domain-informed yet static sampling strategy distributes points uniformly along the $x$- and $y$-dimensions, while along the $z$-dimension it concentrates points predominantly near the top surface with a coarser distribution elsewhere. In contrast, our proposed method samples more heavily in high-residual regions, particularly during early time steps and near the boundaries. This behavior aligns with the findings in Figure~\ref{fig:dev-main}: at an early time step ($t = 0.5$\,s), when heat from laser irradiation has not yet fully dissipated toward the boundaries, collocation points are concentrated primarily near the melt pool. As time progresses and the laser traverses the track, heat propagates throughout the domain, and high-residual regions accumulate in the far-field subsurface where heat dissipates rather than remaining confined to the vicinity of the laser track. Accordingly, the points generated by our adaptive proposal shift toward lower $z$-levels rather than remaining restricted to the top surface.

Most importantly, the distribution of these collocation points depends strongly on the thermophysical properties of each material, as evidenced by the varying marginal distributions across materials. As thermal conductivity increases, the marginal distributions along the $x$- and $y$-dimensions become smoother, reflecting the enhanced heat diffusivity that requires allocating more points farther from the laser track. Similarly, for the $z$-marginal distribution, lower-diffusivity materials retain points mainly near the top surface ($z = 6$\,mm), while higher thermal diffusivity drives the point concentration deeper toward the bottom, where heat accumulates and dissipates to form high-residual zones, a trend further corroborated by Figure~\ref{fig:dev-main}.

\paragraph{Ablation Study and Sensitivity Analysis}
To validate our domain-informed adaptive sampling strategy, we conduct an ablation study and sensitivity analysis on the mixture formulation $q_\alpha^{b}(x\mid t,\mu) := (1-\alpha)\,g_\phi(x\mid t,\mu) + \alpha\,b(x\mid t)$ by varying the mixture weight $\alpha \in [0,1]$ and comparing our domain-informed base distribution $b(x\mid t)$ against an uninformative uniform base.

\begin{figure}[thb]
\centering
\includegraphics[width=0.5\textwidth]{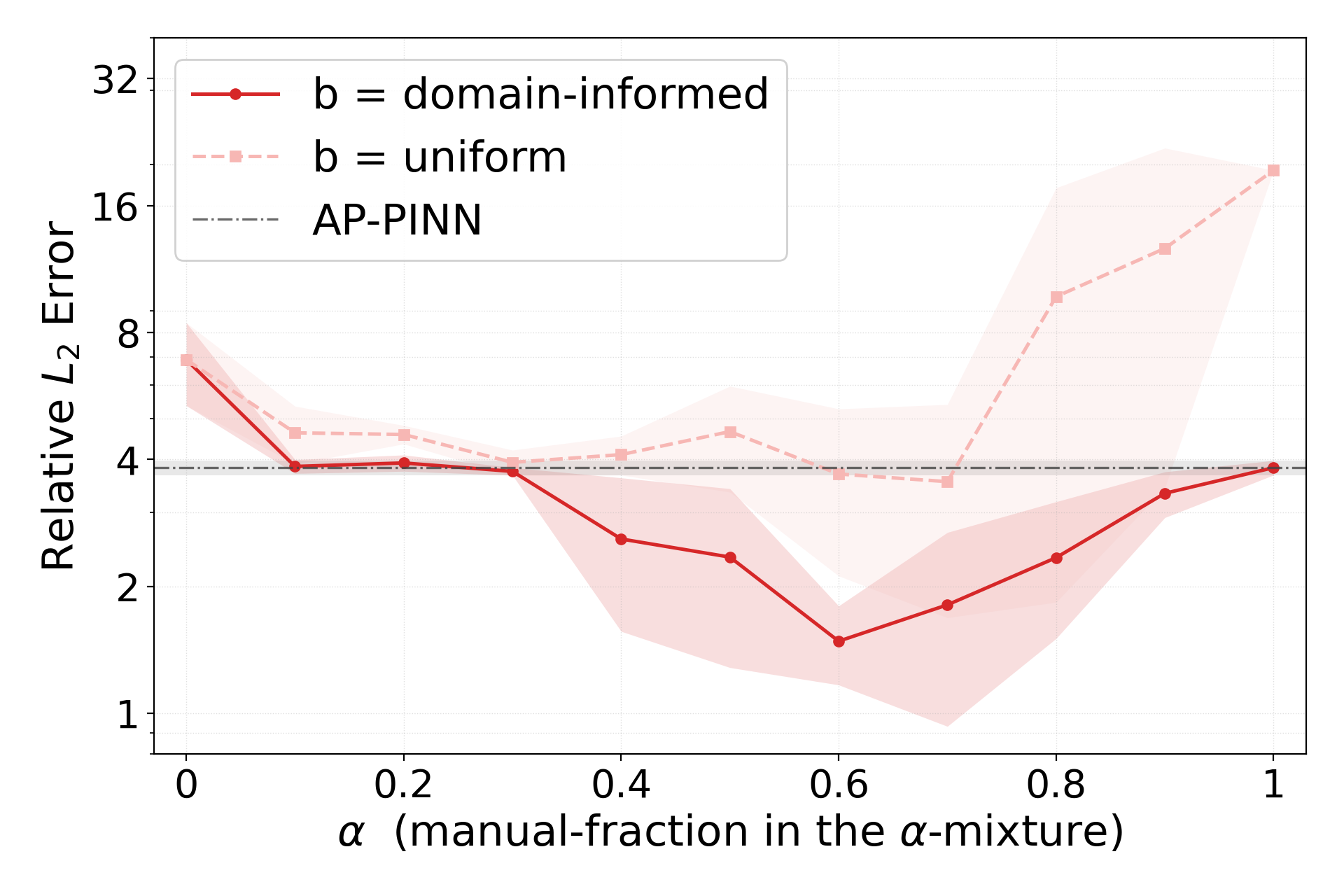}
\caption{Relative $L_2$ error (mean $\pm$ std over five seeds) across the mixture weight $\alpha$ for the two base distributions.}
\label{fig:alpha-sweep}
\end{figure}

As shown in Figure~\ref{fig:alpha-sweep}, the domain-informed base consistently outperforms the uniform base across nearly all values of $\alpha$. The proposed method achieves its lowest relative $L_2$ error on a broad plateau over $\alpha \in [0.5, 0.8]$, reaching a minimum at $\alpha = 0.6$ ($1.48 \pm 0.31\%$). This wide plateau highlights low sensitivity to $\alpha$, which stems from the complementary balance between the two sampling components: the learned proposal $g_\phi$ concentrates on heat-dissipating regions, while the base distribution ensures coverage of highly transient regions near the laser path. Conversely, setting $\alpha = 0$ (relying solely on the learned proposal) leads to degraded accuracy due to the lack of coverage in the transient regions.

Crucially, the base distribution governs training stability under the low-regularity conditions of metal AM. While the uniform base yields somewhat comparable performance for $\alpha \le 0.7$, its performance degrades rapidly at higher $\alpha$, resulting in frequent convergence failures ($\alpha=0.9$) and complete collapse ($\alpha=1$) due to cold-starting with uniformly sampled collocation points under steep thermal gradients. In contrast, the domain-informed base completely eliminates these bimodal failure modes as in previous research. This confirms that incorporating domain knowledge via the base distribution is essential to simultaneously achieve peak accuracy and guaranteed training stability.

\begin{figure}[thbp]
\centering
\begin{subfigure}{\textwidth}
  \includegraphics[width=\textwidth]{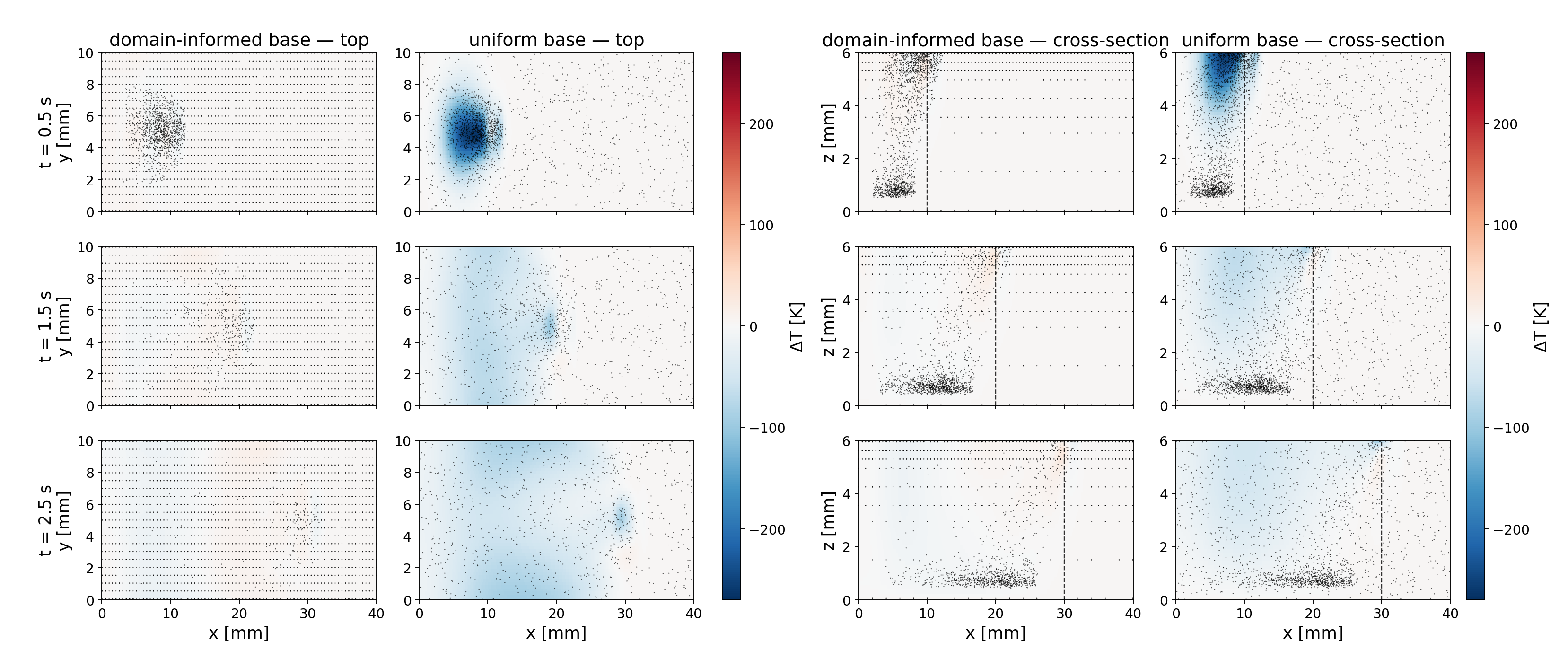}
  \caption{Inconel 718}
  \label{fig:base-dev-in718}
\end{subfigure}\\[4pt]
\begin{subfigure}{\textwidth}
  \includegraphics[width=\textwidth]{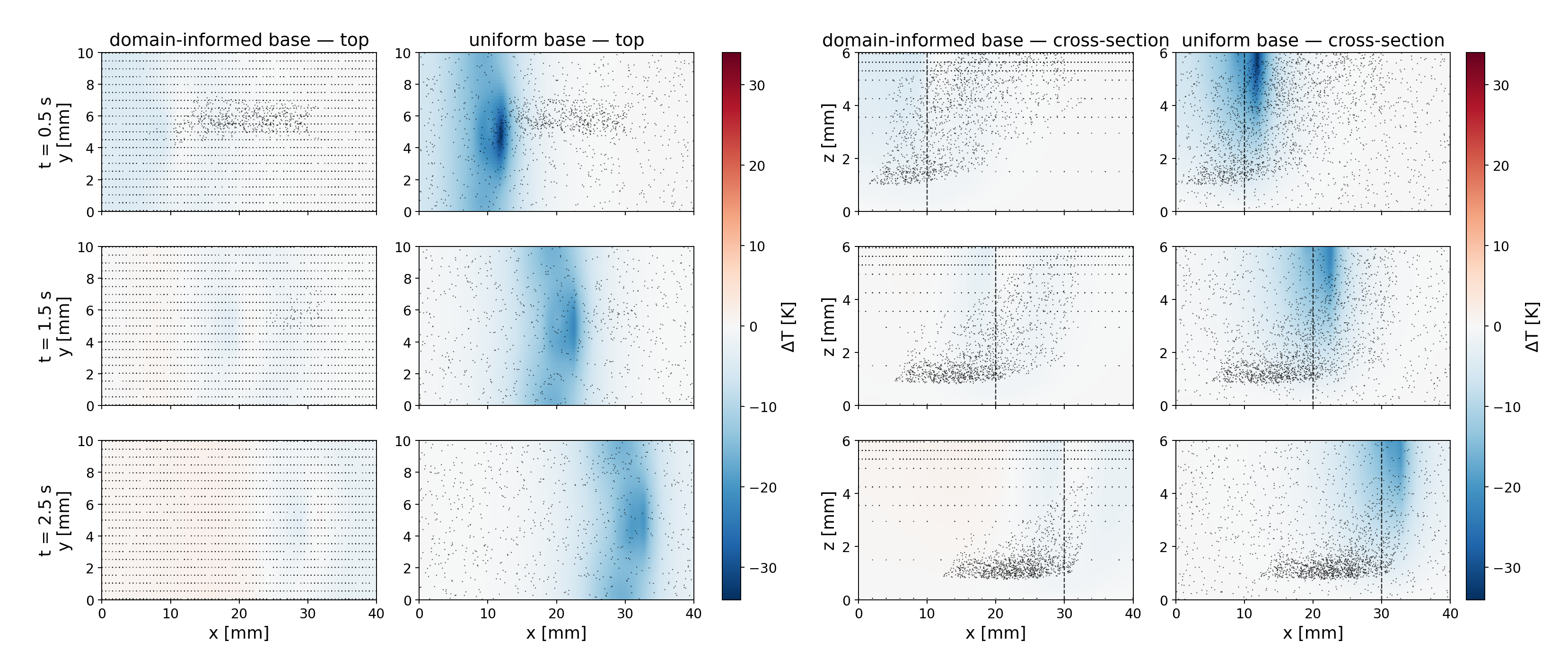}
  \caption{Copper}
  \label{fig:base-dev-cu}
\end{subfigure}
\caption{Signed deviation $\Delta T = T_{\text{PINN}} - T_{\text{FEM}}$ at $t = 0.5$, $1.5$, $2.5$~s for models trained with the domain-informed and the uniform base at the same $\alpha = 0.6$, with each configuration's collocation points overlaid.}
\label{fig:base-dev}
\end{figure}

To physically clarify this performance gap at $\alpha = 0.6$, Figure~\ref{fig:base-dev} compares the deviation fields under both base distributions. The domain-informed base places high sampling density within $1\,\mathrm{mm}$ of the laser track to ensure coverage of highly transient regions, whereas the uniform base severely undersamples this zone, leading to early-time underestimation near the heat source. This spatial undersampling manifests according to material properties: in low-conductivity Inconel~718, errors trail the moving beam along the traversed path, whereas in highly conductive copper ($\approx 35\times$ higher conductivity), errors travel dynamically beneath the laser. By embedding domain knowledge to guarantee coverage near steep thermal gradients, the domain-informed base eliminates these systematic errors, yielding nearly featureless deviation fields across all materials.

\paragraph{Overall discussion}

In summary, our experiments show that the proposed adaptive sampling strategy substantially outperforms both baselines, reducing AP-PINN's relative $L_2$ error by an average of $62.1\%$ and up to $70.1\%$ across all five materials. Despite higher training losses, the method achieves a lower estimated population loss by concentrating collocation points in high-residual regions, which translates into accurate subsurface and far-field thermal predictions where AP-PINN's static sampling fails. This adaptivity is material-dependent, dynamically shifting according to each material's thermophysical properties, and the ablation study confirms that incorporating metal AM domain knowledge is essential for both predictive accuracy and training stability.


While the proposed method shows considerable effectiveness and is, to our knowledge, the first adaptive collocation sampling method introduced in the PINN literature for metal AM, several limitations remain. First, the state-of-the-art baseline and ours still show noticeable deviations near the melt pool under direct laser irradiation due to the extremely high-frequency thermal gradients. Second, the necessity of the domain-informed base distribution $b(x\mid t)$ and the choice of mixture weight $\alpha$ are justified primarily through empirical observations rather than a principled theoretical account. Third, our ground-truth data are obtained from FEM simulations, a widely adopted numerical benchmark in this literature; nonetheless, validation against actual experiments would further strengthen confidence in the method's applicability. Finally, among process conditions we vary only the feedstock material, which aligns with our main objective since thermophysical properties reshape the residual landscape more strongly than process parameters such as laser power or scan speed; still, incorporating such parameters and varying laser scan paths would help establish broader applicability.

These limitations suggest several directions for future work. As discussed earlier, incorporating the temporal dimension into the adaptive sampling scheme, or adopting Fourier-based feature encoding, could better resolve high-frequency thermal transients near the laser track. Developing a theoretical account of the domain-informed base distribution, along with a principled procedure for selecting $\alpha$, would strengthen the method's foundations. Validating predicted thermal histories against actual experiments would further test its applicability beyond FEM-based benchmarks. Finally, extending the framework to jointly vary process parameters and laser scan paths, in addition to feedstock materials, would help establish broader applicability across the full space of AM process conditions.

\section{Conclusion}\label{sec:conclusion}

This work presents, to the best of our knowledge, the first adaptive collocation sampling strategy for PINNs in metal AM. We theoretically show that the generalization error of parametric PINNs is governed by the variance of the empirical loss estimator under the sampling proposal, and that any static, $\mu$-independent proposal faces a strict variance ceiling, making it strictly less effective than adaptive sampling that accounts for process conditions. Guided by this analysis, we adopt Conditional Flow Matching so that a single amortized generator produces the near-optimal proposal for any process condition in one forward pass, and integrate metal AM domain knowledge by mixing its samples with a domain-informed base distribution for training stability.

Extensive experiments on a widely adopted metal AM benchmark demonstrate that the proposed method reduces the relative $L_2$ error of the state-of-the-art parametric PINN baseline by an average of $62.1\%$ and up to $70.1\%$, with less than $1\%$ additional training overhead. Analysis of the empirical and the estimated population loss, thermal history predictions, and qualitative deviation fields further reveals that this improvement stems from the ability of our adaptive proposal to dynamically reallocate collocation points toward heat dissipation regions, which conventional sampling overlooks and therefore fails to capture at increasing depth and distance from the laser path. The accompanying ablation and sensitivity analysis confirms that the domain-informed base distribution is essential for training stability, complementing the learned proposal to jointly achieve peak accuracy and robust convergence.

Taken together, our theoretical and empirical results establish that adaptive, process condition-aware collocation sampling is critical to realizing the generalization potential of parametric PINNs in metal AM. Building on the limitations and future directions discussed above, we hope this work encourages a shift from static, manually crafted collocation strategies toward theoretically grounded, adaptive sampling frameworks for PINNs in metal AM and beyond.

\appendix
\section{Appendix}
\subsection{Nomenclature}\label{app:nomenclature}

\resizebox{\linewidth}{!}{%
\begin{tabular}{llc}
\toprule
Symbol & Meaning & Unit \\
\midrule
\multicolumn{3}{l}{\textit{General notation}} \\[2pt]
$(x, t)$ & spatiotemporal collocation point & \\
$\Omega,\ \Omega_{x},\ t_{\mathrm{end}}$ & spatiotemporal / spatial domain; simulated horizon & \\
$\mu,\ \mathcal{P}$ & process condition $\mu = (\rho_c, C_p, k_c)$; its (compact) parameter space & \\
$p$ & target density (uniform over $\Omega\times\mathcal{P}$) & \\
$q,\ q^\star_\theta,\ b,\ q_\alpha^{b}$ & general proposal; variance-optimal proposal; domain-informed base; mixture & \\
$g_\phi$ & learned CFM generator approximating $q^\star_{\theta_{\mathrm{src}}}$ & \\
$\alpha$ & mixture weight & \\
$f_\theta$ & predictor network with learnable parameters $\theta$ & \\
$r_{\theta,\mathrm{pde}},\ r_{\theta,\mathrm{bc}},\ r_{\theta,\mathrm{ic}}$ & PDE / BC / IC residuals & \\
$\lambda_{\mathrm{pde}},\ \lambda_{\mathrm{bc}},\ \lambda_{\mathrm{ic}},\ R$ & residual balancing weights; residual bound & \\
$\mathcal{L}(\theta),\ \widehat{\mathcal{L}}_N(\theta;q)$ & population loss; empirical loss estimator & \\
$M_q,\ \sigma_q^2(\theta),\ \bar\sigma_q^2$ & likelihood-ratio bound; estimator variance; its supremum over $\Theta$ & \\
$V_{\mathrm{c}}^\star,\ V_{\mathrm{nc}}^\star$ & minimal variance of conditional / unconditional proposal families & \\
$\tau,\ v_\phi,\ z_\tau,\ \ell$ & flow-matching time; velocity field; flow state; residual-group identifier & \\
$\varrho_\theta$ & normalized spatial shape of the squared residual (Appendix~\ref{prf:rem-strict-sep}) & \\
$N,\ G,\ M_{\mathrm{res}},\ B,\ J,\ N_{\mathrm{ODE}}$ & collocation batch; pool size; resample size; CFM batch; stored conditions; Euler steps & \\
$K_{\mathrm{warm}},\ K_{\mathrm{flow}},\ K_{\mathrm{PINN}},\ \theta_{\mathrm{src}}$ & stage iteration counts; warm-start (source) predictor parameters & \\[2pt]
\midrule
\multicolumn{3}{l}{\textit{Physical quantities}} \\[2pt]
$x$ & spatial coordinate & $\mathrm{m}$ \\
$t$ & physical time & $\mathrm{s}$ \\
$T,\ T_\theta$ & ground-truth / predicted temperature field & $\mathrm{K}$ \\
$T_0$ & ambient (initial) temperature & $\mathrm{K}$ \\
$\rho_c$ & density & $\mathrm{kg/m^3}$ \\
$C_p$ & specific heat capacity & $\mathrm{J/(kg\cdot K)}$ \\
$k_c$ & thermal conductivity & $\mathrm{W/(m\cdot K)}$ \\
$P$ & laser power & $\mathrm{W}$ \\
$\eta$ & laser absorptivity & -- \\
$R_{\mathrm{laser}}$ & laser beam radius & $\mathrm{m}$ \\
$x_{\mathrm{laser}}(t),\ \mathbf{v},\ v$ & laser center position; scan velocity; scan speed & $\mathrm{m}$; $\mathrm{m/s}$; $\mathrm{m/s}$ \\
$h$ & convective heat transfer coefficient & $\mathrm{W/(m^2\cdot K)}$ \\
$e$ & surface emissivity & -- \\
$\sigma$ & Stefan--Boltzmann constant & $\mathrm{W/(m^2\cdot K^4)}$ \\
$q_{\mathrm{laser}},\ q_{\mathrm{conv}},\ q_{\mathrm{rad}}$ & laser / convective / radiative heat flux & $\mathrm{W/m^2}$ \\
\bottomrule
\end{tabular}%
}

\subsection{Governing Equations, Boundary Conditions, and Residuals}
\label{app:governing_eqs}

The spatial domain is $\Omega_{x}\subset\mathbb{R}^{d_{x}}$ with $d_{x}=3$, the time interval is $[0,t_{\mathrm{end}}]$, and $\Omega=\Omega_{x}\times[0,t_{\mathrm{end}}]$. The PDE parameter vector collects the three thermophysical material properties that are varied in our experiments \( \mu=(\rho_c,\,C_p,\,k_c)\in\mathcal{P}\subset\mathbb{R}^{m}\), namely the density $\rho_c$, the specific heat capacity $C_p$, and the thermal conductivity $k_c$. All remaining quantities appearing below, including the process parameters and the surface exchange coefficients, are held fixed.

\paragraph{Governing PDE.}
Transient heat transfer in the bare-plate metal AM process is governed by conservation of
energy coupled with Fourier's law of heat conduction. The exact temperature field
$T(x,t;\mu)$ satisfies
\begin{equation}\label{eq:app_pde}
\mathcal{N}[T](x,t;\mu)
:= \rho_c\,C_p\,\frac{\partial T}{\partial t}
- \nabla\!\cdot\!\bigl(k_c\,\nabla T\bigr) = 0,
\qquad (x,t)\in\Omega_{x}\times(0,t_{\mathrm{end}}],
\end{equation}
where $\mathcal{N}[\cdot]$ is the nonlinear differential operator introduced in
Sec.~\ref{sec:setup}; the nonlinearity enters through the radiative boundary term below.

\paragraph{Boundary conditions.}
The boundary $\partial\Omega_{x}$ is decomposed into the top surface
$\Gamma_{\mathrm{top}}$, the side surfaces $\Gamma_{\mathrm{side}}$, and the bottom surface
$\Gamma_{\mathrm{bottom}}$, so that
$\partial\Omega_{x}=\Gamma_{\mathrm{top}}\cup\Gamma_{\mathrm{side}}\cup\Gamma_{\mathrm{bottom}}$.
With $\mathbf{n}$ the outward unit normal, the boundary operator
$\mathcal{B}[\cdot]$ is defined piecewise as
\begin{equation}\label{eq:app_bc}
\mathcal{B}[T](x,t;\mu):=
\begin{cases}
-k_c\,\nabla T\cdot\mathbf{n}
-\bigl(q_{\mathrm{laser}}+q_{\mathrm{conv}}+q_{\mathrm{rad}}\bigr),
& x\in\Gamma_{\mathrm{top}},\\[4pt]
-k_c\,\nabla T\cdot\mathbf{n}
-\bigl(q_{\mathrm{conv}}+q_{\mathrm{rad}}\bigr),
& x\in\Gamma_{\mathrm{side}},\\[4pt]
T(x,t;\mu)-T_0,
& x\in\Gamma_{\mathrm{bottom}},
\end{cases}
\end{equation}
for $t\in(0,t_{\mathrm{end}}]$, and the boundary conditions read
$\mathcal{B}[T](x,t;\mu)=0$ on $\partial\Omega_{x}\times(0,t_{\mathrm{end}}]$.

The laser heat flux imposed on the top surface is modeled as a Gaussian surface heat source,
\begin{equation}\label{eq:app_laser}
q_{\mathrm{laser}}(x,t)
= \frac{2\,\eta\,P}{\pi R_{\mathrm{laser}}^{2}}
\exp\!\left(
-\frac{2\,\lVert x-x_{\mathrm{laser}}(t)\rVert^{2}}{R_{\mathrm{laser}}^{2}}
\right),
\end{equation}
where $\eta$ is the laser absorptivity, $P$ the laser power, and $R_{\mathrm{laser}}$ the
laser beam radius. The laser center follows the prescribed scan path
$x_{\mathrm{laser}}(t)=x_0+\mathbf{v}\,t$; a one-dimensional scan along the
$x$-axis on the top surface is considered, i.e., $\mathbf{v}=(v,0,0)$.
Convective and radiative fluxes act on all surfaces except
$\Gamma_{\mathrm{bottom}}$, where a fixed-temperature Dirichlet condition is imposed:
\begin{align}
q_{\mathrm{conv}}(x,t;\mu) &= h\bigl(T(x,t;\mu)-T_0\bigr),
\label{eq:app_conv}\\
q_{\mathrm{rad}}(x,t;\mu)  &= \sigma\,e\bigl(T(x,t;\mu)^{4}-T_0^{4}\bigr),
\label{eq:app_rad}
\end{align}
with $h$ the convective heat transfer coefficient, $\sigma$ the Stefan--Boltzmann constant,
and $e$ the surface emissivity. Note that $q_{\mathrm{laser}}$ is independent of $\mu$,
whereas $q_{\mathrm{conv}}$ and $q_{\mathrm{rad}}$ inherit the parameter dependence of the
temperature field. The quartic term in \eqref{eq:app_rad} is the source of the nonlinearity
of the boundary value problem.

\paragraph{Initial condition.}
\begin{equation}\label{eq:app_ic}
\mathcal{I}[T](x;\mu) := T(x,0;\mu) - T_0 = 0,
\qquad x\in\Omega_{x},
\end{equation}
with $T_0$ the ambient temperature.

\paragraph{Residuals.}
Substituting the network approximation $T_\theta(x,t;\mu)$ into
\eqref{eq:app_pde}--\eqref{eq:app_bc} yields the three residual terms appearing in
Eq.~\eqref{eq:total_residual}:
\begin{align}
r_{\theta,\mathrm{pde}}(x,t;\mu)
&:= \mathcal{N}[T_\theta](x,t;\mu)
= \rho_c C_p \frac{\partial T_\theta}{\partial t}
- \nabla\!\cdot\!\bigl(k_c \nabla T_\theta\bigr),
&& (x,t)\in\Omega_{x}\times(0,t_{\mathrm{end}}],
\label{eq:app_r_pde}\\[4pt]
r_{\theta,\mathrm{bc}}(x,t;\mu)
&:= \mathcal{B}[T_\theta](x,t;\mu),
&& (x,t)\in\partial\Omega_{x}\times(0,t_{\mathrm{end}}],
\label{eq:app_r_bc}\\[4pt]
r_{\theta,\mathrm{ic}}(x;\mu)
&:= \mathcal{I}[T_\theta](x;\mu) = T_\theta(x,0;\mu)-T_0,
&& x\in\Omega_{x}.
\label{eq:app_r_ic}
\end{align}
Explicitly, the boundary residual $r_{\theta,\mathrm{bc}}$ combines a Neumann-type residual on
$\Gamma_{\mathrm{top}}\cup\Gamma_{\mathrm{side}}$ and a Dirichlet-type residual on
$\Gamma_{\mathrm{bottom}}$:
\begin{equation}\label{eq:app_r_bc_explicit}
r_{\theta,\mathrm{bc}}=
\begin{cases}
-k_c\,\nabla T_\theta\cdot\mathbf{n}
-q_{\mathrm{laser}}-q_{\mathrm{conv}}[T_\theta]-q_{\mathrm{rad}}[T_\theta],
& x\in\Gamma_{\mathrm{top}},\\[3pt]
-k_c\,\nabla T_\theta\cdot\mathbf{n}
-q_{\mathrm{conv}}[T_\theta]-q_{\mathrm{rad}}[T_\theta],
& x\in\Gamma_{\mathrm{side}},\\[3pt]
T_\theta(x,t;\mu)-T_0,
& x\in\Gamma_{\mathrm{bottom}}.
\end{cases}
\end{equation}
Each residual is evaluated on collocation points drawn from the corresponding subdomain, and
the three terms are combined with the balancing weights
$\lambda_{\mathrm{pde}},\lambda_{\mathrm{bc}},\lambda_{\mathrm{ic}}$ as in
Eq.~\eqref{eq:total_residual}. Note that $r_{\theta,\mathrm{ic}}$ is independent of $t$ since
it is evaluated at $t=0$, and it depends on $\mu$ only through $T_\theta$ itself.

\subsection{Proof of Excess Generalization Error Decomposition Eq.~\eqref{eq:uniform-dev}}\label{prf:excgen}

We derive the decomposition of the excess generalization error given in \eqref{eq:uniform-dev} based on the triangular inequality in statistical learning theory. Recall that $\theta^\star \in \arg\min_{\theta \in \Theta} \mathcal{L}(\theta)$ is the population loss minimizer, $\widehat\theta_N(q) \in \arg\min_{\theta \in \Theta} \widehat{\mathcal{L}}_N(\theta;q)$ is the empirical loss minimizer under proposal $q$, and $\widetilde\theta_N$ is the parameter returned after training. The excess generalization error is decomposed as follows:
\begin{align}
\underbrace{\mathcal{L}(\widetilde\theta_N) - \mathcal{L}(\theta^\star)}_{\text{excess generalization error}}
&= \big[ \mathcal{L}(\widetilde\theta_N) - \widehat{\mathcal{L}}_N(\widetilde\theta_N; q) \big] 
+ \big[ \widehat{\mathcal{L}}_N(\widetilde\theta_N; q) - \widehat{\mathcal{L}}_N(\widehat\theta_N(q); q) \big] \nonumber \\
&\quad + \big[ \widehat{\mathcal{L}}_N(\widehat\theta_N(q); q) - \widehat{\mathcal{L}}_N(\theta^\star; q) \big] 
+ \big[ \widehat{\mathcal{L}}_N(\theta^\star; q) - \mathcal{L}(\theta^\star) \big].
\label{eq:decomp_full}
\end{align}

\noindent Since $\widehat\theta_N(q)$ is the minimizer of $\widehat{\mathcal{L}}_N(\cdot;q)$ over $\Theta$, and $\theta^\star \in \Theta$, we have $\widehat{\mathcal{L}}_N(\widehat\theta_N(q); q) \le \widehat{\mathcal{L}}_N(\theta^\star; q)$, which implies:
\begin{equation}
\widehat{\mathcal{L}}_N(\widehat\theta_N(q); q) - \widehat{\mathcal{L}}_N(\theta^\star; q) \;\le\; 0.
\end{equation}

\noindent Applying this inequality to \eqref{eq:decomp_full}, we obtain:
\begin{align}
\underbrace{\mathcal{L}(\widetilde\theta_N) - \mathcal{L}(\theta^\star)}_{\text{excess generalization error}} 
&\le \big[ \widehat{\mathcal{L}}_N(\widetilde\theta_N; q) - \widehat{\mathcal{L}}_N(\widehat\theta_N(q); q) \big] 
+ \big[ \mathcal{L}(\widetilde\theta_N) - \widehat{\mathcal{L}}_N(\widetilde\theta_N; q) \big] 
+ \big[ \widehat{\mathcal{L}}_N(\theta^\star; q) - \mathcal{L}(\theta^\star) \big].
\end{align}

\noindent Finally, bounding the last two terms by their supremum over the hypothesis class $\Theta$, we get:
\begin{align}
\underbrace{\mathcal{L}(\widetilde\theta_N) - \mathcal{L}(\theta^\star)}_{\text{excess generalization error}} 
&\le \big[ \widehat{\mathcal{L}}_N(\widetilde\theta_N; q) - \widehat{\mathcal{L}}_N(\widehat\theta_N(q); q) \big] 
+ \sup_{\theta \in \Theta} \big| \widehat{\mathcal{L}}_N(\theta; q) - \mathcal{L}(\theta) \big| 
+ \sup_{\theta \in \Theta} \big| \widehat{\mathcal{L}}_N(\theta; q) - \mathcal{L}(\theta) \big| \nonumber \\
&= \underbrace{\big[ \widehat{\mathcal{L}}_N(\widetilde\theta_N; q) - \widehat{\mathcal{L}}_N(\widehat\theta_N(q); q) \big]}_{\text{optimization error}} 
+ 2 \, \sup_{\theta \in \Theta} \underbrace{\big| \widehat{\mathcal{L}}_N(\theta; q) - \mathcal{L}(\theta) \big|}_{\text{statistical error}}
\end{align} \qed

\subsection{Proof of Proposition~\ref{prop:genbound}}\label{prf:genbound}

Prior to the proofs hereafter, we summarize the structural assumptions used throughout our analysis:

\begin{assumption}[Regularity]\label{asm:reg}
The target density $p$ is a probability density on $\Omega \times \mathcal{P}$ and strictly positive almost everywhere on its support.
\end{assumption}

\begin{assumption}[Residual Boundedness]\label{asm:bound}
There exists a constant $R < \infty$ such that $\sup_{\theta \in \Theta} |r_\theta(x,t;\mu)| \leq R$ almost everywhere on $\Omega \times \mathcal{P}$.
\end{assumption}

\begin{assumption}[Proposal Compatibility]\label{asm:supp}
The proposal density $q$ satisfies $\mathrm{supp}(p) \subseteq \mathrm{supp}(q)$, and the likelihood ratio is bounded almost everywhere, i.e., $M_q := \mathrm{ess\,sup}_q \frac{p(x,t,\mu)}{q(x,t,\mu)} = \mathrm{ess\,sup}_q \frac{p(x)}{q(x\mid t,\mu)} < \infty$.
\end{assumption}

\noindent Building on these assumptions, we state and prove the lemma regarding the variance of the empirical loss estimator:

\begin{lemma}[Variance of the empirical loss]\label{lem:variance}
Under Assumptions~\ref{asm:reg}--\ref{asm:supp}, for every $\theta \in \Theta$, the variance of the empirical loss estimator $\widehat{\mathcal{L}}_N(\theta;q)$ satisfies:
\begin{equation}\label{eq:var-decomp}
\mathrm{Var}_q\!\left[\widehat{\mathcal{L}}_N(\theta;q)\right]
= \frac{1}{N}\,\sigma_q^2(\theta),
\qquad
\sigma_q^2(\theta)
:= \mathbb{E}_q\!\left[\left(\frac{p(x)}{q(x \mid t,\mu)}\,r_\theta(x,t;\mu)^2\right)^{\!2}\,\right] - \mathcal{L}(\theta)^2.
\end{equation}
\end{lemma}

\begin{proof}[Proof of Lemma~\ref{lem:variance}]
Since $\widehat{\mathcal{L}}_N(\theta;q) = \frac{1}{N}\sum_{i=1}^N \frac{p(x_i)}{q(x_i \mid t_i, \mu_i)}r_\theta(x_i,t_i;\mu_i)^2$ is an average of $N$ i.i.d. random variables, its variance is $\frac{1}{N}$ times the single-sample variance:
\begin{align}
\sigma_q^2(\theta) 
&= \mathrm{Var}_q\left[\frac{p(x)}{q(x \mid t, \mu)}\,r_\theta(x,t;\mu)^2\right] \nonumber \\
&= \mathbb{E}_q\!\left[\left(\frac{p(x)}{q(x \mid t,\mu)}\,r_\theta(x,t;\mu)^2\right)^{\!2}\,\right] - \left(\mathbb{E}_q\!\left[\frac{p(x)}{q(x \mid t,\mu)}\,r_\theta(x,t;\mu)^2\right]\right)^2.
\end{align}

\noindent Expanding the first-moment expectation inside the second term gives:
\begin{align}
\mathbb{E}_q\!\left[\frac{p(x)}{q(x \mid t,\mu)}\,r_\theta(x,t;\mu)^2\right]
&= \int_{\mathcal{P}} \int_0^{t_{\mathrm{end}}} \int_{\Omega_{x}} \frac{p(x)}{q(x \mid t,\mu)}\,r_\theta(x,t;\mu)^2 \cdot \Big(q(x \mid t, \mu) p(t) p(\mu)\Big) \, dx\, dt\, d\mu \nonumber \\
&= \int_{\mathcal{P}} \int_0^{t_{\mathrm{end}}} \int_{\Omega_{x}} p(x)\, p(t)\, p(\mu) \cdot r_\theta(x,t;\mu)^2 \, dx\, dt\, d\mu \nonumber \\
&= \mathbb{E}_{p}\!\left[r_\theta(x,t;\mu)^2\right] = \mathcal{L}(\theta).
\end{align}

\noindent Plugging this back into the variance expression, the second term simplifies to $\mathcal{L}(\theta)^2$, which yields:
\begin{equation}\label{eq:app_lemma1}
\sigma_q^2(\theta) = \mathbb{E}_q\!\left[\left(\frac{p(x)}{q(x \mid t,\mu)}\,r_\theta(x,t;\mu)^2\right)^{\!2}\,\right] - \mathcal{L}(\theta)^2.
\end{equation}
Multiplying by $\frac{1}{N}$ completes the proof of Lemma~\ref{lem:variance}.
\end{proof}

\vspace{4mm}

\begin{proof}[Proof of Proposition~\ref{prop:genbound}]
The proof proceeds in three intuitive steps: establishing a pointwise concentration bound for a fixed $\theta$, uniformizing the bound over the parameter space $\Theta$, and isolating the precise dependence on the proposal distribution $q$.

\paragraph{Step 1: Pointwise Concentration via Bernstein's Inequality.}
Fix any parameter vector $\theta \in \Theta$. For the $i$-th collocation point $(x_i, t_i, \mu_i) \sim q$, define the importance-weighted random variable:
\begin{equation}
Z_i(\theta) := \frac{p(x_i, t_i, \mu_i)}{q(x_i, t_i, \mu_i)} \, r_\theta(x_i, t_i; \mu_i)^2 = \frac{p(x_i)}{q(x_i \mid t_i, \mu_i)} \, r_\theta(x_i, t_i; \mu_i)^2.
\end{equation}
By definition, the Monte Carlo estimator $\widehat{\mathcal{L}}_N(\theta; q) = \frac{1}{N}\sum_{i=1}^N Z_i(\theta)$ is an unbiased estimator of the population loss, i.e., $\mathbb{E}_q[Z_i(\theta)] = \mathcal{L}(\theta)$. 

We characterize two essential moment properties of $Z_i(\theta)$:
\begin{enumerate}[label=(\roman*)]
    \item \textbf{Per-sample Variance:} $\mathrm{Var}_q\big(Z_i(\theta)\big) = \sigma_q^2(\theta)$, as established in Lemma~\ref{lem:variance}.
    \item \textbf{Essential Boundedness:} By Assumptions~\ref{asm:bound} and \ref{asm:supp},
    \begin{equation}
    0 \le Z_i(\theta) \le \mathrm{ess\,sup}_q\left(\frac{p(x)}{q(x \mid t, \mu)}\right) \cdot \sup_{(x,t,\mu)} r_\theta(x, t; \mu)^2 \le M_q R^2 =: U_q.
    \end{equation}
\end{enumerate}
Applying Bernstein's inequality to the i.i.d. sequence $\{Z_i(\theta) - \mathcal{L}(\theta)\}_{i=1}^N$, we obtain for any $\epsilon > 0$:
\begin{equation}\label{eq:pointwise_bernstein}
\mathbb{P}_q \Big( \big|\widehat{\mathcal{L}}_N(\theta; q) - \mathcal{L}(\theta)\big| > \epsilon \Big) \;\le\; 2 \exp \left( - \frac{N \epsilon^2}{2 \sigma_q^2(\theta) + \frac{2}{3} U_q \epsilon} \right).
\end{equation}

\paragraph{Step 2: Uniform Generalization over $\Theta$.}
To extend this concentration uniformly over all $\theta \in \Theta$, we control the capacity of the hypothesis function class $\mathcal{F}_\Theta := \{ (x, t, \mu) \mapsto \frac{p(x)}{q(x \mid t, \mu)} r_\theta(x, t; \mu)^2 : \theta \in \Theta \}$. 

By standard uniform concentration techniques, for any specified confidence parameter $\delta \in (0, 1)$, there exist a function class complexity metric $\mathcal{C}(\Theta, \delta) > 0$ and universal numerical constants $C_1, C_2 > 0$ such that, with probability at least $1-\delta$:
\begin{equation}\label{eq:Bernstein_derived}
\sup_{\theta \in \Theta} \big|\widehat{\mathcal{L}}_N(\theta; q) - \mathcal{L}(\theta)\big| \;\le\; C_1 \sqrt{\frac{\bar{\sigma}_q^2 \, \mathcal{C}(\Theta, \delta)}{N}} \;+\; C_2 \frac{M_q R^2 \, \mathcal{C}(\Theta, \delta)}{N},
\end{equation}
where $\bar{\sigma}_q^2 := \sup_{\theta \in \Theta} \sigma_q^2(\theta)$ represents the worst-case per-sample variance across $\Theta$.

\paragraph{Step 3: Isolating the Dependence on Proposal $q$.}
Note that the complexity term $\mathcal{C}(\Theta, \delta)$ is evaluated under the intrinsic target measure $p$, rendering $\mathcal{C}(\Theta, \delta)$ and the constants $C_1, C_2$ completely independent of $q$. 

The choice of proposal distribution $q$ affects the upper bound strictly through:
\begin{enumerate}
    \item The leading variance term $\bar{\sigma}_q^2$,
    \item The secondary likelihood ratio bound $M_q$.
\end{enumerate}
In practical training regimes where the sample size $N$ satisfies $N \gtrsim \frac{M_q R^2 \mathcal{C}}{\bar{\sigma}_q^2}$, the $\mathcal{O}(1/\sqrt{N})$ term dominates the $\mathcal{O}(1/N)$ term. Thus, minimizing the estimator variance $\bar{\sigma}_q^2$ tightens the leading term of the generalization bound, which completes the proof.
\end{proof}

\subsection{Proof of Theorem~\ref{thm:strict-sep}.}
\label{prf:strict-sep}

We recall from \eqref{eq:app_lemma1} that since $\mathcal{L}(\theta)^2$ does not depend on the proposal choice $q$, minimizing $\sigma_q^2(\theta)$ is equivalent to minimizing the second-moment term $\mathbb{E}_q\!\left[\left(\frac{p(x)}{q(x \mid t,\mu)}\,r_\theta(x,t;\mu)^2\right)^{\!2}\right]$.

\paragraph{Case 1: Conditional Spatial Proposal ($q \in \mathcal{Q}_c$)}
For $q \in \mathcal{Q}_c$, the spatial proposal $q(x \mid t, \mu)$ depends on both time $t$ and parameter $\mu$. Expanding the second moment explicitly:
\begin{align}
\mathbb{E}_q\!\left[\left(\frac{p(x)}{q(x \mid t,\mu)}\,r_\theta(x,t;\mu)^2\right)^{\!2}\,\right]
&= \int_{\mathcal{P}} p(\mu) \int_{0}^{t_{\mathrm{end}}} p(t) \int_{\Omega_{x}} q(x\mid t,\mu) \cdot \frac{p(x)^2}{q(x \mid t,\mu)^2}\,r_\theta(x,t;\mu)^4\, dx\, dt\, d\mu \nonumber \\
&= \mathbb{E}_{t,\mu}\Bigg[\int_{\Omega_{x}} \frac{p(x)^2\,r_\theta(x,t;\mu)^4}{q(x \mid t,\mu)}\, dx\Bigg]. \label{eq:second_moment_c}
\end{align}
Since $\int_{\Omega_{x}} q(x \mid t, \mu) \, dx = 1$ for any $(t, \mu)$, we multiply the integrand by this identity:
\begin{align}
\mathbb{E}_q\!\left[\left(\frac{p(x)}{q(x \mid t,\mu)}\,r_\theta(x,t;\mu)^2\right)^{\!2}\,\right]
&= \mathbb{E}_{t,\mu}\Bigg[\left(\int_{\Omega_{x}} \frac{p(x)^2\,r_\theta(x,t;\mu)^4}{q(x \mid t,\mu)}\, dx\right) \left(\int_{\Omega_{x}} q(x \mid t, \mu) \, dx\right)\Bigg] \nonumber \\
&= \mathbb{E}_{t,\mu}\Bigg[\left(\int_{\Omega_{x}} \left(\frac{p(x)\,r_\theta(x,t;\mu)^2}{\sqrt{q(x \mid t,\mu)}}\right)^2\, dx\right) \left(\int_{\Omega_{x}} \left(\sqrt{q(x \mid t,\mu)}\right)^2\, dx\right)\Bigg].
\end{align}
Applying the Cauchy-Schwarz inequality to the inner spatial integral for each fixed $(t, \mu)$:
\begin{align}
\mathbb{E}_q\!\left[\left(\frac{p(x)}{q(x \mid t,\mu)}\,r_\theta(x,t;\mu)^2\right)^{\!2}\,\right]
&\ge \mathbb{E}_{t,\mu}\Bigg[\left(\int_{\Omega_{x}} \frac{p(x)\,r_\theta(x,t;\mu)^2}{\sqrt{q(x \mid t,\mu)}} \cdot \sqrt{q(x \mid t,\mu)}\, dx\right)^2\Bigg] \nonumber \\
&= \mathbb{E}_{t,\mu}\Bigg[\left(\int_{\Omega_{x}} p(x)\,r_\theta(x,t;\mu)^2\, dx\right)^2\Bigg] \nonumber \\
&= \mathbb{E}_{t,\mu}\Bigg[\Big(\mathbb{E}_{x}[r_\theta(x,t;\mu)^2]\Big)^2\Bigg]. \label{eq:lower_bound_c}
\end{align}
Equality in Cauchy-Schwarz holds if and only if $q(x\mid t,\mu) \propto p(x)r_\theta(x,t;\mu)^2$. Thus, taking the infimum over all $q \in \mathcal{Q}_c$, we get $V_{\mathrm{c}}^\star(\theta) := \inf_{q \in \mathcal{Q}_c} \sigma_q^2(\theta)$:
\begin{align}
V_{\mathrm{c}}^\star(\theta) + \mathcal{L}(\theta)^2 
&= \mathbb{E}_{t,\mu}\Bigg[\Big(\mathbb{E}_{x}[r_\theta(x,t;\mu)^2]\Big)^2\Bigg] \label{eq:Vc_star_def}
\end{align}

\paragraph{Case 2: Unconditional Spatial Proposal ($q \in \mathcal{Q}_{\mathrm{nc}}$)}
For $q \in \mathcal{Q}_{\mathrm{nc}}$, the spatial proposal $q(x \mid t)$ is restricted to be independent of $\mu$. Starting from the second moment in \eqref{eq:second_moment_c}, since $q(x\mid t)$ is independent of $\mu$, we can pass the expectation over $\mu$ inside the spatial integral:
\begin{align}
\mathbb{E}_q\!\left[\left(\frac{p(x)}{q(x \mid t)}\,r_\theta(x,t;\mu)^2\right)^{\!2}\,\right]
&= \mathbb{E}_{t}\Bigg[\int_{\Omega_{x}} \frac{p(x)^2\,\mathbb{E}_{\mu}[r_\theta(x,t;\mu)^4]}{q(x \mid t)}\, dx\Bigg].
\end{align}
Using the identity $\int_{\Omega_{x}} q(x \mid t) \, dx = 1$ for each $t$:
\begin{align}
\mathbb{E}_q\!\left[\left(\frac{p(x)}{q(x \mid t)}\,r_\theta(x,t;\mu)^2\right)^{\!2}\,\right]
&= \mathbb{E}_{t}\Bigg[\left(\int_{\Omega_{x}} \frac{p(x)^2\,\mathbb{E}_{\mu}[r_\theta(x,t;\mu)^4]}{q(x \mid t)}\, dx\right) \left(\int_{\Omega_{x}} q(x \mid t) \, dx\right)\Bigg] \nonumber \\
&= \mathbb{E}_{t}\Bigg[\left(\int_{\Omega_{x}} \left(\frac{p(x)\sqrt{\mathbb{E}_{\mu}[r_\theta(x,t;\mu)^4]}}{\sqrt{q(x \mid t)}}\right)^2 dx\right) \left(\int_{\Omega_{x}} \left(\sqrt{q(x \mid t)}\right)^2 dx\right)\Bigg].
\end{align}
Applying the Cauchy-Schwarz inequality over $\Omega_{x}$ for each $t$:
\begin{align}
\mathbb{E}_q\!\left[\left(\frac{p(x)}{q(x \mid t)}\,r_\theta(x,t;\mu)^2\right)^{\!2}\,\right]
&\ge \mathbb{E}_{t}\Bigg[\left(\int_{\Omega_{x}} p(x)\sqrt{\mathbb{E}_{\mu}[r_\theta(x,t;\mu)^4]}\, dx\right)^2\Bigg] \nonumber \\
&= \mathbb{E}_{t}\Bigg[\left(\mathbb{E}_{x}\Big[\sqrt{\mathbb{E}_{\mu}[r_\theta(x,t;\mu)^4]}\Big]\right)^2\Bigg]. \label{eq:lower_bound_nc}
\end{align}
Equality holds when $q(x\mid t) \propto p(x)\sqrt{\mathbb{E}_{\mu}[r_\theta(x,t;\mu)^4]}$. Hence, taking the infimum over all $q \in \mathcal{Q}_{\mathrm{nc}}$, we obtain $V_{\mathrm{nc}}^\star(\theta) := \inf_{q \in \mathcal{Q}_{\mathrm{nc}}} \sigma_q^2(\theta)$:
\begin{align}
V_{\mathrm{nc}}^\star(\theta) + \mathcal{L}(\theta)^2
&= \mathbb{E}_{t}\Bigg[\left(\mathbb{E}_{x}\Big[\sqrt{\mathbb{E}_{\mu}[r_\theta(x,t;\mu)^4]}\Big]\right)^2\Bigg] \label{eq:Vnc_star_def}
\end{align}

\paragraph{Comparison via Minkowski's Integral Inequality}
To compare $V_{\mathrm{c}}^\star(\theta)$ and $V_{\mathrm{nc}}^\star(\theta)$, we express the inner parameter expectations using the $L^2(\mu)$-norm of the squared residual $r_\theta(x,t;\mu)^2$:
\begin{equation}
\Big\| r_\theta(x,t;\mu)^2 \Big\|_{L^2(\mu)} := \sqrt{\mathbb{E}_{\mu}\left[r_\theta(x,t;\mu)^4\right]}.
\end{equation}

\noindent Rewriting \eqref{eq:Vc_star_def} and \eqref{eq:Vnc_star_def} side by side:
\begin{equation}\label{eq:sidebyside}
V_{\mathrm{c}}^\star(\theta) + \mathcal{L}(\theta)^2 
= \mathbb{E}_t\Bigg[\;\left\|\,\mathbb{E}_{x}\!\left[r_\theta(x,t;\mu)^2\right]\right\|_{L^2(\mu)}^2\;\Bigg],
\qquad
V_{\mathrm{nc}}^\star(\theta) + \mathcal{L}(\theta)^2 
= \mathbb{E}_t\Bigg[\;\left(\mathbb{E}_{x}\!\left[\,\left\|r_\theta(x,t;\mu)^2\right\|_{L^2(\mu)}\right]\right)^2\;\Bigg].
\end{equation}

\noindent The only difference between $V_{\mathrm{c}}^\star$ and $V_{\mathrm{nc}}^\star$ is the order of the spatial expectation $\mathbb{E}_{x}$ and the $L^2(\mu)$-norm: taking the spatial average first and then measuring the parameter norm, versus measuring the parameter norm first and then taking the spatial average.

\noindent By Minkowski's integral inequality:
\begin{equation}\label{eq:minkowski}
\left\|\,\mathbb{E}_{x}\!\left[r_\theta(x,t;\mu)^2\right]\right\|_{L^2(\mu)} 
\;\le\; 
\mathbb{E}_{x}\!\left[\,\left\|r_\theta(x,t;\mu)^2\right\|_{L^2(\mu)}\right].
\end{equation}

\noindent Squaring both sides of \eqref{eq:minkowski} and taking the expectation w.r.t time $t$ yields:
\begin{equation}
\mathbb{E}_t\Bigg[\;\left\|\,\mathbb{E}_{x}\!\left[r_\theta(x,t;\mu)^2\right]\right\|_{L^2(\mu)}^2\;\Bigg] 
\;\le\; 
\mathbb{E}_t\Bigg[\;\left(\mathbb{E}_{x}\!\left[\,\left\|r_\theta(x,t;\mu)^2\right\|_{L^2(\mu)}\right]\right)^2\;\Bigg].
\end{equation}

\noindent Subtracting $\mathcal{L}(\theta)^2$ from both sides gives the desired inequality:
\begin{equation}
V_{\mathrm{c}}^\star(\theta) \;\le\; V_{\mathrm{nc}}^\star(\theta) \qquad \text{for every } \theta \in \Theta.
\end{equation}
\qed

\subsection{Proof of Remark~\ref{rem:strict-sep}.}
\label{prf:rem-strict-sep}

We first recall the equality condition for the continuous Minkowski inequality \eqref{eq:minkowski}. In general, equality in the vector triangle inequality $\|f+g\| = \|f\| + \|g\|$ holds if and only if $f$ and $g$ are linearly dependent with non-negative scalars (i.e., they point in the same direction). 

Extending this to the continuous Minkowski inequality in the $L^2(\mu)$-norm over the parameter space $\mathcal{P}$, equality holds in \eqref{eq:minkowski} at a fixed time $t$ if and only if all spatial functions in the family $\{ x \mapsto r_\theta(x, t; \mu)^2 \}_{\mu \in \mathcal{P}}$ are non-negative scalar multiples of one single spatial function $a_t(x)$. That is, for almost every $t$, there must exist a spatial factor $a_t(x) \ge 0$ and a $\mu$-dependent factor $c_t(\mu) \ge 0$ such that the squared residual factors into a product:
\begin{equation}\label{eq:separable}
r_\theta(x,t;\mu)^2 \;=\; a_t(x)\, c_t(\mu).
\end{equation}

To see how condition \eqref{eq:separable} relates to the spatial shape of the residual, consider its normalized spatial distribution:
\begin{equation}\label{eq:shape_proof}
\varrho_\theta(x \mid t, \mu) 
\;:=\; \frac{p(x)\, r_\theta(x,t;\mu)^2}{\int_{\Omega_{x}} p(x')\, r_\theta(x',t;\mu)^2\, dx'}.
\end{equation}
Substituting the factorized form \eqref{eq:separable} into both the numerator and the denominator of \eqref{eq:shape_proof} yields:
\begin{align}
\varrho_\theta(x \mid t, \mu) 
&\;=\; \frac{p(x)\, \big(a_t(x)\, c_t(\mu)\big)}{\int_{\Omega_{x}} p(x')\, \big(a_t(x')\, c_t(\mu)\big)\, dx'} \\[6pt]
&\;=\; \frac{c_t(\mu) \cdot \big(p(x)\, a_t(x)\big)}{c_t(\mu) \cdot \int_{\Omega_{x}} p(x')\, a_t(x')\, dx'} \\[6pt]
&\;=\; \frac{p(x)\, a_t(x)}{\int_{\Omega_{x}} p(x')\, a_t(x')\, dx'}.
\end{align}
Since the $\mu$-dependent scalar $c_t(\mu)$ cancels out entirely, the resulting density depends solely on $x$ and $t$, making it completely independent of $\mu$. 

Conversely, if $\varrho_\theta(x \mid t, \mu)$ is independent of $\mu$, we can write $\varrho_\theta(x \mid t, \mu) = \varrho_\theta(x \mid t)$. Rearranging definition \eqref{eq:shape_proof} then gives:
\begin{equation}
r_\theta(x,t;\mu)^2 \;=\; \frac{\varrho_\theta(x \mid t)}{p(x)} \cdot \left( \int_{\Omega_{x}} p(x')\, r_\theta(x',t;\mu)^2\, dx' \right),
\end{equation}
which is precisely in the separable form $a_t(x)\, c_t(\mu)$ with $a_t(x) = \varrho_\theta(x \mid t)/p(x)$ and $c_t(\mu) = \int_{\Omega_{x}} p(x')\, r_\theta(x',t;\mu)^2\, dx'$.

Thus, equality holds in Theorem~\ref{thm:strict-sep} if and only if the normalized spatial shape $\varrho_\theta(x \mid t, \mu)$ carries no actual $\mu$-dependence.

\medskip
\noindent This mathematical equivalence highlights a critical physical insight: it is \emph{not} sufficient for the residual magnitude to vary with $\mu$. Any overall magnitude variation is entirely absorbed by $c_t(\mu)$, which an optimally tuned $\mu$-independent proposal can handle without performance degradation. For strict separation to fail (and equality to hold), the residual must maintain a fixed spatial shape across all $\mu$.

\begin{assumption}[Non-degeneracy]\label{asm:nondeg}
For every $\theta \in \Theta$, the normalized spatial shape $\varrho_\theta(\cdot \mid t, \mu)$ is not $p(\mu)$-a.e.\ constant in $\mu$ on a set of times $t$ with positive $p(t)$-measure.
\end{assumption}

\noindent Under Assumption~\ref{asm:nondeg}, the condition for equality \eqref{eq:separable} is violated on a set of times of positive measure. Consequently, the Minkowski inequality \eqref{eq:minkowski} becomes strict, and taking the expectation over $t$ preserves this strict inequality:
\begin{equation}
V_{\mathrm{c}}^\star(\theta) \;<\; V_{\mathrm{nc}}^\star(\theta) \qquad \text{for every } \theta \in \Theta.
\end{equation}

\medskip
\noindent In metal AM thermal modeling, Assumption~\ref{asm:nondeg} naturally holds due to the underlying physical dynamics of the system. Variations in process conditions $\mu$ directly change the governing PDEs and BCs. As a result, the entire \emph{residual landscape} changes its spatial structure rather than just scaling up or down: localized high-residual regions shift or reshape across the domain as $\mu$ varies. Since a spatial profile cannot be expressed as a simple scalar multiple of a shifted or reshaped profile, the normalized shape \eqref{eq:shape_proof} inherently varies with $\mu$. This structural non-separability holds true regardless of the neural network's architecture or approximation quality, as the $\mu$-dependent spatial variation is directly imprinted on the residual landscape by the physical problem itself. This completes the proof of Remark~\ref{rem:strict-sep}.
\qed

\subsection{Qualitative temperature fields comparison}\label{app:qualit}

\begin{figure}[htbp]
  \centering
  \includegraphics[width=\linewidth]{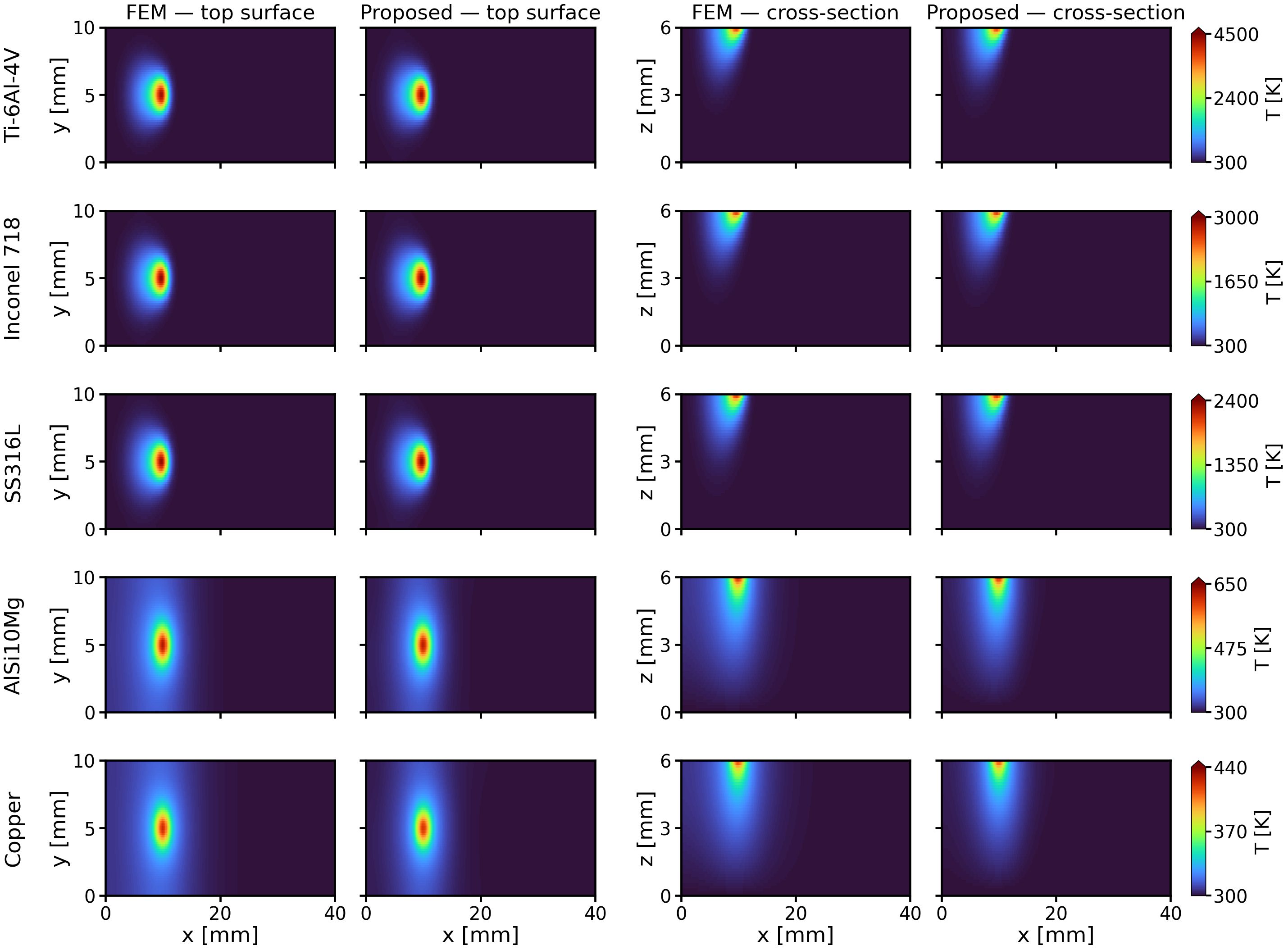}
  \caption{Temperature field at $t = 0.5$~s for the five validation materials (rows): FEM reference and the proposed method, on the top surface and on the $y = 5$~mm cross-section. Note that each material uses its own color scale due to differences in thermophysical properties.}
  \label{fig:tfield_t05}
\end{figure}

\begin{figure}[htbp]
  \centering
  \includegraphics[width=\linewidth]{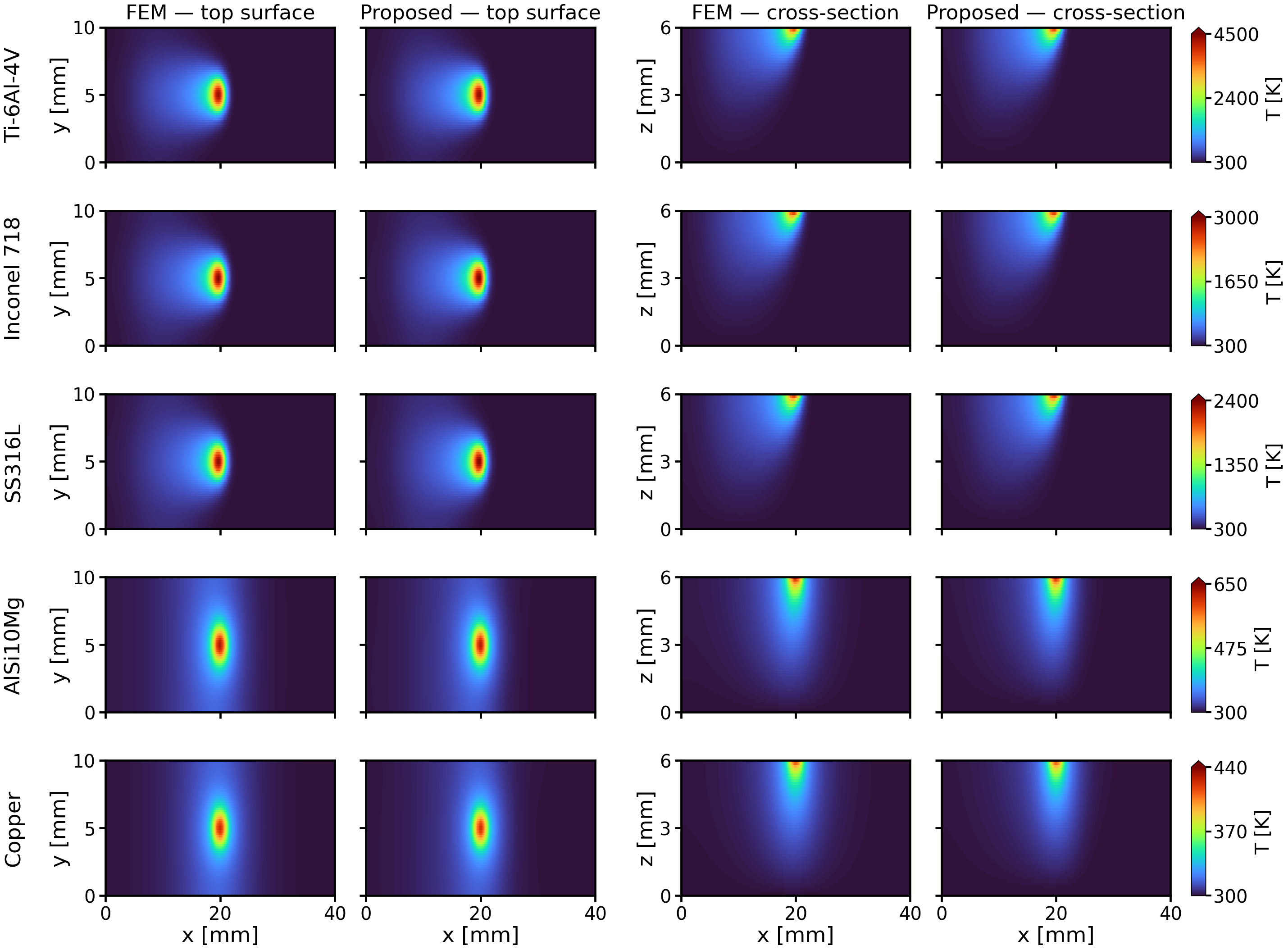}
  \caption{Temperature field at $t = 1.5$~s (same layout and color scales as Figure~\ref{fig:tfield_t05}). The laser has advanced to $x = 20$~mm and the trailing heat tail has developed.}
  \label{fig:tfield_t15}
\end{figure}

\begin{figure}[htbp]
  \centering
  \includegraphics[width=\linewidth]{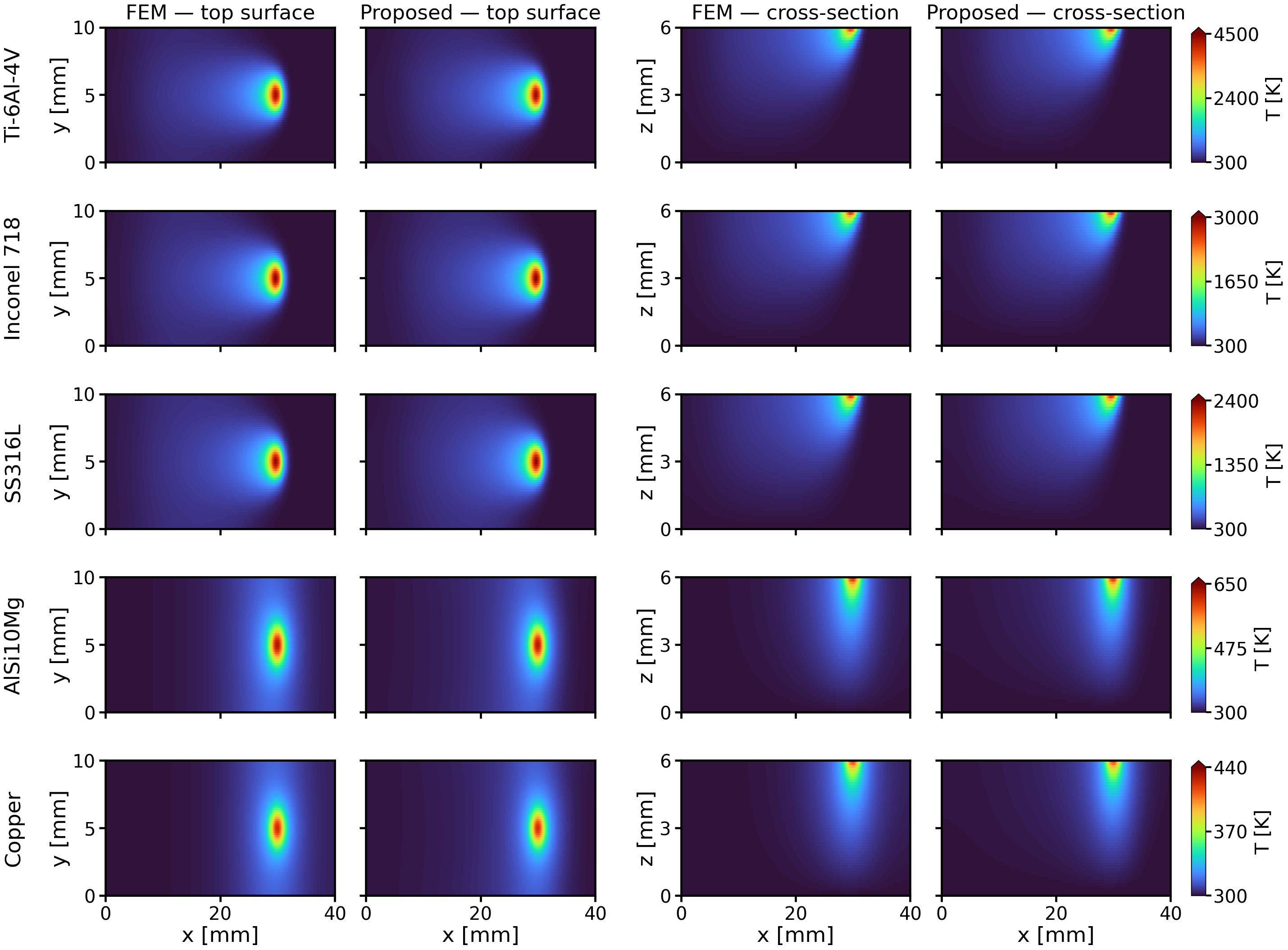}
  \caption{Temperature field at $t = 2.5$~s (same layout and color scales as Figure~\ref{fig:tfield_t05}). The proposed method reproduces the melt pool, the trailing heat tail, and the depth of heat penetration for every material, including the out-of-distribution AlSi10Mg and Copper.}
  \label{fig:tfield_t25}
\end{figure}

\subsection{Temperature fields deviation from FEM}\label{app:dev}

\begin{figure}[htbp]
\centering
\includegraphics[width=\textwidth]{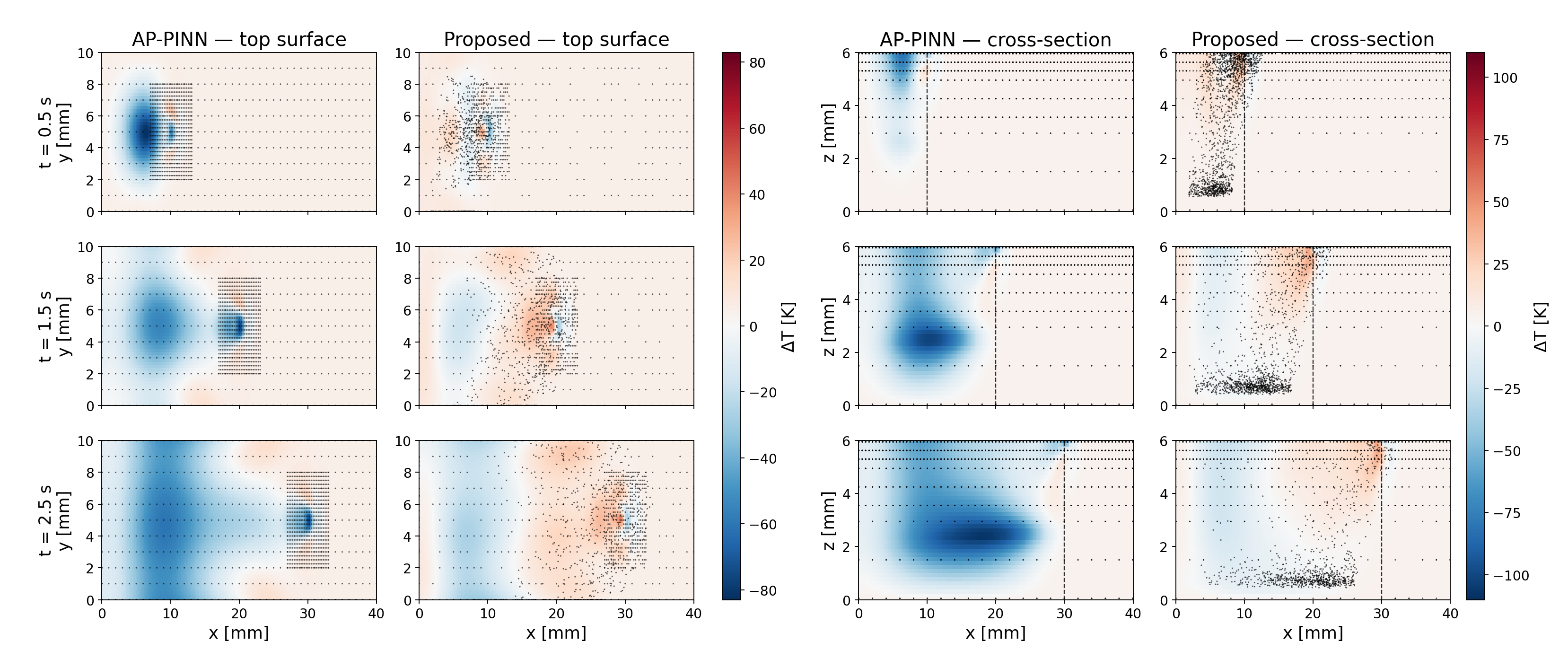}
\caption{Signed deviation and collocation overlay for Ti-6Al-4V
($k_c = 6.7$~W\,m$^{-1}$\,K$^{-1}$), plotted as in Fig.~\ref{fig:dev-main}.}
\label{fig:dev-ti64}
\end{figure}

\begin{figure}[htbp]
\centering
\includegraphics[width=\textwidth]{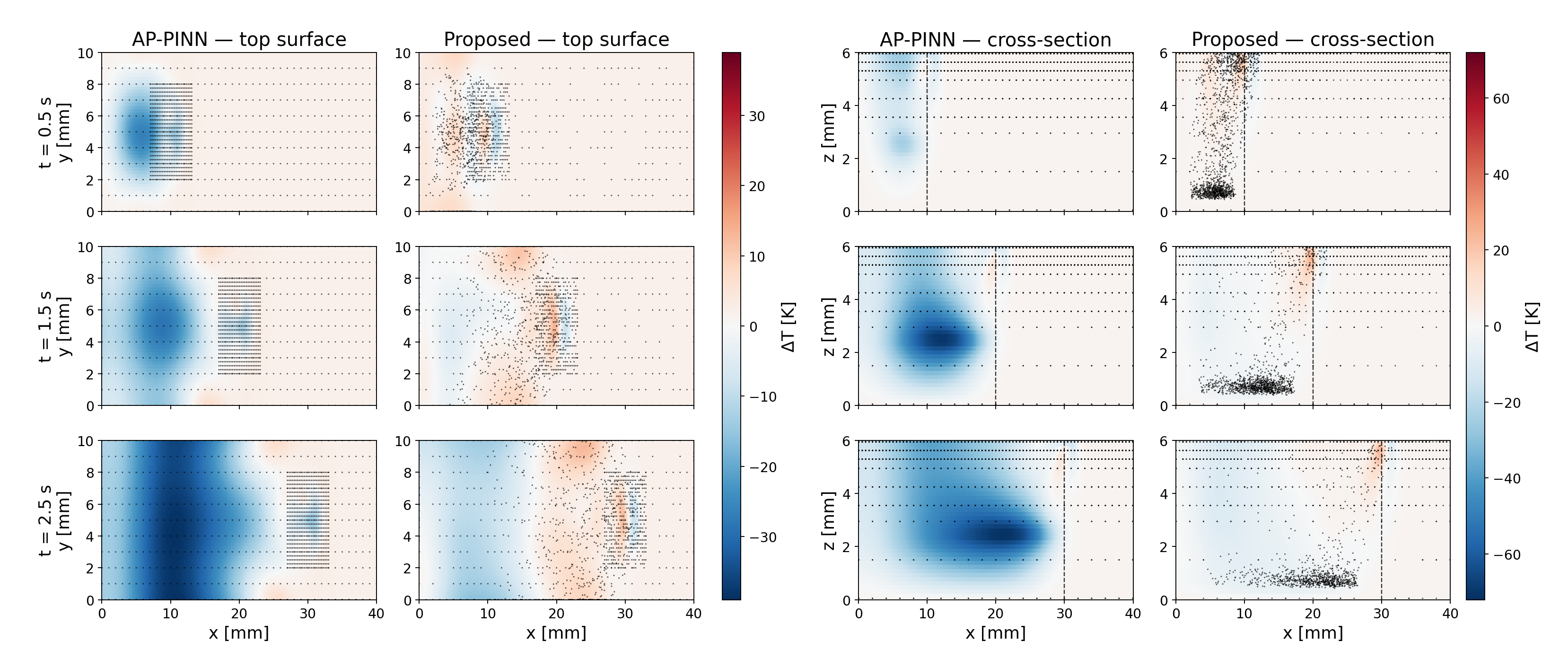}
\caption{Signed deviation and collocation overlay for SS316L
($k_c = 16$~W\,m$^{-1}$\,K$^{-1}$), plotted as in Fig.~\ref{fig:dev-main}.}
\label{fig:dev-ss316l}
\end{figure}

\begin{figure}[htbp]
\centering
\includegraphics[width=\textwidth]{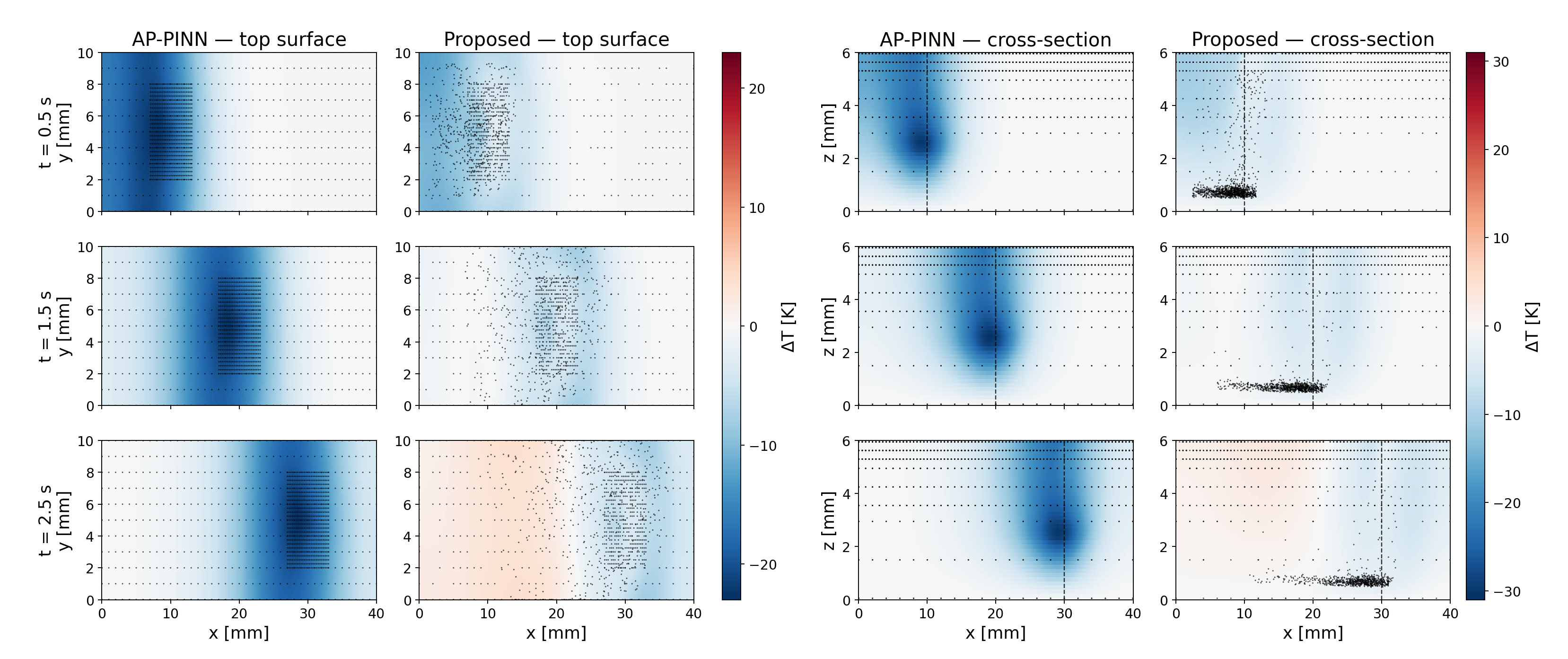}
\caption{Signed deviation and collocation overlay for AlSi10Mg
($k_c = 150$~W\,m$^{-1}$\,K$^{-1}$), plotted as in Fig.~\ref{fig:dev-main}.}
\label{fig:dev-alsi}
\end{figure}

\section*{Acknowledgements}
This work was supported by the National Research Foundation of Korea (NRF) grant funded by the Korea government (MSIT) (RS-2025-02216260, RS-2024-00431837).

\section*{Data availability statement}
The data and code in this study are available at \url{https://github.com/hsleecri/AdaptiveSamplingPINN-AM}.

\section*{Generative Artificial Intelligence (AI)}
The authors used Claude Sonnet 5 by Anthropic for grammar checking, polishing, and improving readability. After using this tool, the authors reviewed and edited the content as needed, and acknowledge that the authors take full responsibility for the content of the publication.
\bibliographystyle{chicago}
\bibliography{IISE-Trans}
\end{document}